\documentclass{article}

\usepackage[accepted]{icml2020}

\usepackage[utf8]{inputenc} 
\usepackage[T1]{fontenc}    
\usepackage{hyperref}       
\usepackage{url}            

\usepackage{amsmath,amsfonts,amssymb,amsthm,bm,accents,dsfont,bbm}
\usepackage{graphicx,enumerate,url}
\graphicspath{{Images/}}

\usepackage{blindtext}

\newcommand{\executeiffilenewer}[3]{%
\ifnum\pdfstrcmp{\pdffilemoddate{#1}}%
{\pdffilemoddate{#2}}>0%
{\immediate\write18{#3}}\fi%
}

\usepackage[english]{babel}

\usepackage{color}

\newtheorem{theorem}{Theorem}
\newtheorem{proposition}{Proposition}

\newtheorem{lemma}{Lemma}
\newtheorem{definition}{Definition}
\theoremstyle{definition}
\newtheorem{remark}{Remark}
\newtheorem{assumption}{Assumption}

\newcommand{\norm}[1]{\ensuremath{\left\| #1 \right\|}}
\newcommand{\abs}[1]{\ensuremath{{\left\vert #1 \right\vert}}}

\DeclareMathOperator*{\argmax}{argmax}
\DeclareMathOperator*{\minimize}{minimize}

\DeclareMathOperator{\subjectto}{subject\ to}

\newcommand{\calA}{\ensuremath{\mathcal{A}}}

\newcommand{\calC}{\ensuremath{\mathcal{C}}}
\newcommand{\calD}{\ensuremath{\mathcal{D}}}

\newcommand{\calK}{\ensuremath{\mathcal{K}}}
\newcommand{\calL}{\ensuremath{\mathcal{L}}}

\newcommand{\calN}{\ensuremath{\mathcal{N}}}

\newcommand{\bzero}{\ensuremath{\bm{0}}}

\newcommand{\bI}{\ensuremath{\bm{I}}}

\newcommand{\bW}{\ensuremath{\bm{W}}}

\newcommand{\bSigma}{\ensuremath{\bm{\Sigma}}}

\newcommand{\bc}{\ensuremath{\bm{c}}}

\newcommand{\bx}{\ensuremath{\bm{x}}}

\newcommand{\bz}{\ensuremath{\bm{z}}}

\newcommand{\blambda}{\ensuremath{\bm{\lambda}}}

\newcommand{\bbR}{\ensuremath{\mathbb{R}}}

\newcommand{\bcb}{\ensuremath{\bm{{\bar{c}}}}}

\newcommand{\bxb}{\ensuremath{\bm{{\bar{x}}}}}

\newcommand{\bzb}{\ensuremath{\bm{{\bar{z}}}}}

\newcommand{\bch}{\ensuremath{\bm{{\hat{c}}}}}

\newcommand{\bxh}{\ensuremath{\bm{{\hat{x}}}}}

\def\st/{\textsuperscript{st}}
\def\nd/{\textsuperscript{nd}}
\def\rd/{\textsuperscript{rd}}
\def\th/{\textsuperscript{th}}

\newcommand{\setR}{\bbR}

\newcommand{\zeros}{\ensuremath{\bm{0}}}
\newcommand{\ones}{\ensuremath{\bm{\mathds{1}}}}

\def\nnil{\nil}
\newcounter{prob}
\newenvironment{prob}[1][\nil]{%
	\def\tmp{#1}
	\equation
	\ifx\tmp\nnil
		\refstepcounter{prob}
		\tag{P\Roman{prob}}
	\else
		\tag{\tmp}
	\fi
	\aligned%
}{%
	\endaligned\endequation%
}

\newenvironment{prob*}{%
	\csname equation*\endcsname%
	\aligned%
}{%
	\endaligned%
	\csname endequation*\endcsname%
}

\icmltitlerunning{Trust but Verify: Assigning Prediction Credibility by Counterfactual Constrained Learning}

\begin{document}

\twocolumn[
\icmltitle{Trust but Verify:\\Assigning Prediction Credibility by Counterfactual Constrained Optimization}


\icmlsetsymbol{equal}{*}

\begin{icmlauthorlist}
\icmlauthor{Luiz F. O. Chamon}{penn}
\icmlauthor{Santiago Paternain}{penn}
\icmlauthor{Alejandro Ribeiro}{penn}
\end{icmlauthorlist}

\icmlaffiliation{penn}{Department of Electrical and Systems Engineering, University of Pennsylvania, Philadelphia, PA, USA}

\icmlcorrespondingauthor{Luiz F. O. Chamon}{luizf@seas.upenn.edu}


\vskip 0.3in
]

\printAffiliationsAndNotice{}

\begin{abstract}

Prediction credibility measures, in the form of confidence intervals or probability distributions, are fundamental in statistics and machine learning to characterize model robustness, detect out-of-distribution samples~(outliers), and protect against adversarial attacks. To be effective, these measures should (i)~account for the wide variety of models used in practice, (ii)~be computable for trained models or at least avoid modifying established training procedures, (iii)~forgo the use of data, which can expose them to the same robustness issues and attacks as the underlying model, and (iv)~be followed by theoretical guarantees. These principles underly the framework developed in this work, which expresses the credibility as a risk-fit trade-off, i.e., a compromise between how much can fit be improved by perturbing the model input and the magnitude of this perturbation~(risk). Using a constrained optimization formulation and duality theory, we analyze this compromise and show that this balance can be determined counterfactually, without having to test multiple perturbations. This results in an unsupervised, \emph{a posteriori} method of assigning prediction credibility for any~(possibly non-convex) differentiable model, from RKHS-based solutions to any architecture of (feedforward, convolutional, graph)~neural network. Its use is illustrated in data filtering and defense against adversarial attacks.

\end{abstract}

\section{Introduction}
\label{S:Intro}

Assessing credibility to predictions is a fundamental problem in statistics and machine learning~(ML) with both practical and societal impact, finding applications in model robustness, outlier detection, defense against adversarial noise, and ML explanation, to name a few. For statistical models, these credibilities often comes in the form of uncertainty measures, such as class probability or confidence intervals, or sensitivity measures, often embodied by the regression coefficients themselves~\citep{Smithson02c, Hawkins80i, Hastie01t}. As models become more involved and opaque, however, their complex input--coefficients--output relation, together with miscalibration and robustness issues, have made obtaining reliable credibility measures increasingly challenging. This fact is illustrated by the abundance of examples of neural networks~(NNs) and non-parametric models that are susceptible to adversarial noise or provide overly confident predictions~\citep{Guo17o, Platt99p, Szegedy14i, Goodfellow14e}. It has become clear that model outputs are not reliable assessments of prediction credibility.

A notable exception is the case of statistical models that embedded credibility measures directly in their structure, such as Bayesian approaches~\citep{Rasmussen05g, Neal96b, MacKay92a, Shridhar19a}. Obtaining these measures, however, requires that such models be fully trained using modified, often computationally intensive procedures. In the context of NNs, for instance, Bayesian methods remain feasible only for moderately sized architectures~\citep{Blundell15w, Shridhar19a, Gal15b, Heek19b}. Even if unlimited compute power were available, access to training data may not be, due to privacy or security concerns. For such settings, specific variational model approximations that preclude retraining have been proposed, such as Monte Carlo dropout~\citep{Gal16d}. Still, despite their empirical success in some applications~\citep{Dhillon18s, Sheikholeslami19e}, they do not approximate the posterior distribution needed to assess uncertainty in the Bayesian framework, i.e., the conditional distribution of the model given the training data.

This work assigns credibility to the model outputs by using their robustness to input perturbations rather then modifying the model structure and/or training procedure. However, instead of assuming a predetermined source of perturbation, e.g., by limiting their magnitude, they are determined based on the underlying as a trade-off between fit and risk. In other words, the input is perturbed to strike a balance between how much its \emph{fit} to any class can be improved and how much this perturbation modifies it~(\emph{risk}). Credible predictions are obtained at the point where the risk of improving fit by further perturbing the original input outweighs the actual improvements. This yields a practical method to assign credibilities that abides by the \emph{desiderata} postulated in abstract, i.e., it is (i)~\emph{general}, as it can be computed for any~(possibly non-convex) differentiable model; (ii)~\emph{a posteriori}, as it can be evaluated for already trained models; (iii)~\emph{unsupervised}, as it does not rely on any form of training data; and (iv)~\emph{theoretically sound}, as it~(locally) solves both a constrained optimization and a Bayesian estimation problem.

More to the point, the main contributions of this work are:

\begin{enumerate}[1.]

	\item formalize credibility as a fit-risk trade-off~(Sections~\ref{S:problem} and~\ref{S:tradeoff}) and show that it is, in fact, the maximum \emph{a posteriori}~(MAP) estimator of a probabilistic failure model~(Section~\ref{S:bayesian});

	\item develop a local theory of compromises that holds for generic, non-convex models and enables fit-risk trade-offs to be determined counterfactually, i.e., without testing multiple perturbations~(Section~\ref{S:compromise});

	\item put forward an algorithm to compute this credibility measure for virtually any pretrained~(possibly non-convex) differentiable model, from RKHS-based methods to~(convolutional) NNs~(CNNs)~(Section~\ref{S:algorithm}).

\end{enumerate}

Numerical experiments are used to illustrate applications of these credibility profiles for a CNN in the contexts of filtering and adversarial noise~(Section~\ref{S:sims}). Before proceeding, we contextualize this problem in a short literature review.

\section{Related work}

Credibility is a central concept in statistics, arising in the form of confidence intervals, probability distributions, or sensitivity analysis. Typically, these metrics are obtained for specific models, such as linear or generalized linear models, and become more intricate to compute as the model complexity increases. Confidence intervals, for instance, can be obtained in closed-form only for simple models for which the asymptotic distribution of error statistics are available~\citep{Smithson02c}. Otherwise, bootstrapping techniques must be used. Though more widely applicable, they require several models to be trained over different data sets, which can be prohibitive in certain large-scale applications~\citep{Efron94a, Dietterich00e, Li18r}. Sensitivity analysis is an alternative that does away with randomness by focusing instead in how much sample point or input values influence fit and/or model coefficients. The leverage score or the linear regression coefficients are a good examples of this approach~\citep{Hawkins80i, Hastie01t}. Bayesian models directly embed uncertainty measures by learning probability distributions instead of point estimates. Though Bayesian inference has been deployed in a wide variety of model classes, learning these models remains challenging, especially in complex, large-scale settings common in deep learning and CNNs~\citep{Friedman01t, Blundell15w, Gal15b, Shridhar19a, Heek19b}. Approximate measures based on specific probabilistic models, e.g., Monte Carlo dropout~\citep{Gal16d}, have been proposed to address these issues. However, they do not approximate the posterior distribution of the model given the training data needed to assess uncertainty in the Bayesian framework. Other empirically motivated credibility scores, e.g., based on $k$-nearest neighbors embeddings in the learned representation space, require retraining with modified cost functions to be effective~\citep{Mandelbaum17d}.

Credibility measures can be used to solve a variety of statistical problems. For instance, they are effective in assessing the robustness and performance of models or detecting out-of-distribution samples or outliers, both during training or deployment~\citep{Chandola09a, Hodge04a, Hawkins80i, Aggarwal15o}. Selective classifiers are often used to this end~\citep{Flores58a, Cortes16b, El-Yaniv10o, Geifman17s}. Realizing that in critical applications it is often better to admit one's limitations than to provide uncertain answers, filtered classifiers are given the additional option of not classifying an input. These can be flagged for human inspection. Their performance is quantified in terms of coverage~(how many samples the model chooses to classify) and filtered accuracy~(accuracy over the classified samples), two quantities typically at odds with each other~\citep{El-Yaniv10o, Geifman17s}. Measures of credibility are also often used in the related context of adversarial noise~\citep{Dhillon18s, Wong18p, Sheikholeslami19e}.

Finally, this work leverages duality to derive both theoretical results and algorithms. It is worth noting that the dual variables of convex optimization programs have a well-known sensitivity interpretation. Indeed, the dual variable corresponding to a constraint determines the variation of the objective if that constraint were to be tightened or relaxed~\citep{Boyd04c, Bonnans00p}. Since this work also considers non-convex models for which these results do not hold, it develops of a novel local sensitivity theory of compromise~(Section~\ref{S:compromise}).

\section{Problem formulation}
\label{S:problem}

This work considers the problem of assigning credibility to the predictions of a classifier. Formally, given a model~$\phi^o: \setR^p \to \setR^m$, a sample~$\bx^o \in \setR^p$, and a set~$\calK$ of possible classes, we wish to assign a number~$c_k \in \setR$ to each class~$k \in \calK$ representing our credence that~$\bx^o$ belongs to~$k$. By credence, we mean that if~$c_{k^\prime} \leq c_k$, then we would rather bet on the proposition ``$x$ belongs to~$k$'' than ``$x$ belongs to~$k^\prime$''. For convenience, we fix an order of~$\calK$ and often refer to the vector~$\bc \in \setR^{\abs{\calK}}$ collecting the~$c_k$ as the credibility profile. We use~$\abs{\calA}$ to denote the cardinality of a set~$\calA$. This problem is therefore equivalent to that of ranking from most to least credible the classes that the input~$\bx^o$ could belong to. Though certain credibility measures impose additional structure, as is the case of probabilities, they are minimally required to induce a weak ordering of the set~$\calK$~\citep{Roberts09a}.

The output of any classifier induces a measure of credibility through the training loss function. Let~$\ell_k: \setR^m \to \setR_+$ denote a loss function with respect to class~$k \in \calK$, such as the logistic negative log-likelihood, the cross-entropy loss, or the hinge loss~\citep{Hastie01t}. Then, the loss of the model output with respect to the class~$k \in \calK$ induces a measure of \emph{in}credibility. Conversely, we can obtain the credibility measure
\begin{equation}\label{E:fitCred}
	c_k = -\ell_k\big( \phi^o(\bx^o) \big)
		\text{,} \quad k \in \calK
		\text{,}
\end{equation}
In this sense, classifying a sample as~$\argmax_{k \in \calK} c_k$~(with ties broken arbitrarily if the set is not a singleton) is the same as assigning it to the most credible class. Note that since losses are non-negative, credibility is a non-positive number that increases as the loss decreases, i.e., as the goodness-of-fit improves.

Nevertheless, \eqref{E:fitCred}---and the model output~$\phi^o(\bx^o)$ for that matter---are generally unreliable credibility metrics. Indeed, since models are optimized by focusing only on the most credible class, they do not incorporate ranking information on the other ones. In fact, classifier outputs often resemble singletons~(see, e.g., Figs.~\ref{F:res_vs_dense}a and~\ref{F:res_vs_dense}c). The resulting miscalibration can skew the credences of all but the top class~\citep{Guo17o, Platt99p}. Though there are exceptions, most notably Bayesian models, their use is not as widespread in many application domains due to their training complexity~\citep{Blundell15w, Shridhar19a, Gal15b, Heek19b}. Fit and model outputs are also susceptible to robustness issues due to overfitting, poor or lacking training data, and/or the use of inadequate parametrizations~(e.g., choice of NN architecture or kernel bandwidth). What is more, the input~$\bx^o$ can be an outlier or have been contaminated by noise, adversarial or not. Complex models such as~(C)NNs have proven to be particularly vulnerable to several of these issues~\citep{Szegedy14i, Goodfellow14e, Dhillon18s, Wong18p}.

Thus, the model output can only provide a sound credibility measure if both the model and the input are reliable. In fact, the issues associated with~\eqref{E:fitCred} all but vanish if the input~$\bx^o$ is trustworthy and the model output~$\phi^o(\bx^o)$ is robust to perturbations. Based on this observation, the next section puts forward a fit-based credibility measure that incorporates robustness to perturbations.

\begin{remark}

It is worth contrasting the problem of assigning credibility to model predictions described in this section to its classical learning version in which credibility is assigned based on a data set~$\calD$ of labeled pairs~$(\bx_n,y_n) \in \setR^p \times \calK$. The current problem is in fact at least as hard as the learning one both computationally and statistically, since it is reducible to the learning case. Additionally, the learning setting is considerably more restrictive in practice. Indeed, it may not be possible, due to privacy or security concerns, to access the raw training data from which a model was obtained. Even if it were, any solution that requires the complete retraining of every deployed model is infeasible in many applications, especially if complex, modified training procedures, such as those used to train Bayesian NNs, must be deployed. This is particularly critical for large-scale, non-convex models.

\end{remark}

\section{Credibility as a fit-risk trade-off}
\label{S:tradeoff}

In Section~\ref{S:problem}, we argued that a loss function induces a credibility measure through~\eqref{E:fitCred}, but that it is reliable only if the classifier output~$\phi^o(\bx^o)$ is insensitive to perturbations. In what follows, we show how we can \emph{filter} the input~$\bx^o$ in order to produce such reliable credibility profiles.

We do not say filter here to mean removing noise or artifacts~(though this may be a byproduct of the procedure), but obtaining a modified input~$\bx^\star$ such that~$\phi^o(\bx^\star)$ is robust to perturbations. However, since our goal is to ultimately use~$\phi^o(\bx^\star)$ to produce credibility profiles, the direction and magnitude of these perturbations is not dictated by a fixed disturbance model as in robust statistics and adversarial noise applications. Instead, we obtain task- and sample-specific settings using the underlying data properties captured by the model~$\phi^o$ during training.

Indeed, we consider perturbations that improve fit~(increase credibility), not with respect to a specific class, but with respect to all classes simultaneously. We do so because we do not seek a single most credible class for the input, but the best credibility profile with respect to all classes. Any perturbation of~$\bx^o$ that worsens the overall fit leads to an altogether less confident classification and thus, to a poor measure of credibility. At the same time, $\bx^\star$ must be anchored at~$\bx^o$: the further we are from the original input, the more we have modified the original classification problem. By allowing arbitrarily large perturbations, the original image of a dog could be entirely replaced by that of a cat, at which point we are no longer assigning credibility to predictions of dog pictures. The magnitude of the perturbation thus relates to the risk of not solving the desired problem%
\footnote{While this work focuses on perturbations of the input~$\bx^o$, similar definitions and results also apply to perturbations of the model~$\phi^o$ with minor modifications. We leave the details of this alternative formulation for future work.}.

In summary, the ``filtered''~$\bx^\star$ must balance two conflicting objectives: improving fit while remaining close to~$\bx^o$. Indeed, while smaller perturbations of~$\bx^o$ are preferable to reduce risk, larger perturbations are allowed if the fit pay-off is of comparable magnitude. We formalize this concept in the following definition, where~$\bW$ is a diagonal, positive definite matrix and~$\norm{\bz}_{\bW}^2 = \bz^T \bW \bz$ is used to denote the $\bW$-weighted Euclidian norm.

\begin{definition}[Credibility profile]\label{D:cred}
Given a model~$\phi^o$ and an input~$\bx^o$, the credibility profile of the predictions~$\phi^o(\bx^o)$ is a vector~$\bc^\star$ that satisfies
\begin{equation}\label{E:compromise}
	\norm{\bx^\star - \bx^o}^2 - \norm{\bx^\prime - \bx^o}^2 \leq
		\norm{\bc^\prime}_{\bW^{-1}}^2 - \norm{\bc^\star}_{\bW^{-1}}^2
		\text{,}
\end{equation}
where the credibility profiles are obtained from~\eqref{E:fitCred} as\\$\bc^\star = \left[ -\ell_k\big(\phi^o(\bx^\star) \big) \right]_{k \in \calK}$ and $\bc^\prime = \left[ -\ell_k\big(\phi^o(\bx^\prime) \big) \right]_{k \in \calK}$.
\end{definition}

In other words, the credibility profile~$\bc^\star$ for the predictions of the model~$\phi^o$ with respect to the input~$\bx^o$ is obtained using the modified input~$\bx^\star$ that has better fit than any other input closer to~$\bx^o$ and is closer to~$\bx^o$ that any input with better overall fit. In particular, note that for~$\bx^\prime = \bx^o$, \eqref{E:compromise} becomes
\begin{equation}\label{E:preCompromise}
	\norm{\bx^\star - \bx^o}^2 \leq \norm{\bc^o}_{\bW^{-1}}^2-\norm{\bc^\star}_{\bW^{-1}}^2
		\text{,}
\end{equation}
where~$\bc^o = \left[ -\ell_k\big( \phi^o(\bx^o) \big) \right]_{k \in \calK}$ are the credences induced by the original input~$\bx^o$. Hence, the perturbation that generates~$\bx^\star$ is at most as large as the overall credence gains~[recall from~\eqref{E:fitCred} that the~$c_k \leq 0$]. These gains are weighted by the matrix~$\bW$ that can be used either to control risk aversion by, e.g., by taking~$\bW = \gamma \bI$ and using~$\gamma$ to control the magnitude of the right-hand side of~\eqref{E:preCompromise}, or embed \emph{a priori} knowledge on the value of the credibilities~(see Section~\ref{S:bayesian}). While it may be easier to interpret, note that~\eqref{E:preCompromise} is not well-posed since it always holds for~$\bx^\star = \bx^o$. Def.~\eqref{D:cred} is therefore strictly stronger as~\eqref{E:compromise} must hold for any reference input~$\bx^\prime$, not only the original~$\bx^o$.

In the next section, we analyze the properties of the compromise~\eqref{E:compromise}. In particular, we are interested in those properties that do not depend on the convexity of~$\phi^o$, so that they are applicable to a wider class of models including, e.g., (C)NNs. We first show that the perturbation needed to achieve a credence profile can be determined by solving a constrained mathematical program~(Section~\ref{S:constrained}). Leveraging this formulation and results from duality theory, we then obtain equivalent formulations of Def.~\ref{D:cred}~(Section~\ref{S:counterfactual}) that are used to show that~$\bc^\star$ is in fact a MAP estimate~(Section~\ref{S:bayesian}) and to provide an algorithm that simultaneously determines~$\bx^\star$ and~$\bc^\star$ counterfactually, i.e., without testing multiple inputs or credences~(Section~\ref{S:algorithm}).

\section{A counterfactual theory of compromise}
\label{S:compromise}

Def.~\ref{D:cred} established credibility in terms of a compromise between fit and risk. It is not immediate, however, whether such a compromise exists or if it can be determined efficiently. The goal of this section is to address these questions by developing a~(local) counterfactual theory of compromise. By counterfactual, we mean that the main result of this section~(Theorem~\ref{T:counterfactual}) characterizes properties of~\eqref{P:primal} that would turn arbitrary credences into credibilities that satisfy~\eqref{E:compromise}. In Section~\ref{S:algorithm}, we leverage this result to put forward an algorithm to directly find local solutions of~\eqref{P:primal} that satisfy~\eqref{E:compromise} without repeatedly testing profiles~$\bc$.

\subsection{A constrained formulation of credibility}
\label{S:constrained}

The first step in our derivations is to express the trade-off~\eqref{E:compromise} solely in terms of credences. To do so, we formalize the relation between a credence profile~$\bc$ and the magnitude~$r^\star(\bc)$ of the smallest perturbation of~$\bx^o$ required to find an input that achieves it. We describe this relation by means of the constrained optimization problem
\begin{prob}\label{P:primal}
	r^\star(\bc) = \min_{\bx \in \setR^p}& &&\norm{\bx - \bx^o}^2
	\\
	\subjectto& &&\ell_k\big( \phi^o(\bx) \big) \leq -c_k
		\text{,} \quad k \in \calK
		\text{.}
\end{prob}
Define~$r^\star(\bc) = +\infty$ if the program is infeasible, i.e., if there exist no~$\bx$ such that~$\ell_k\big( \phi^o(\bx) \big) \leq -c_k$ for all~$k \in \calK$. Problem~\eqref{P:primal} seeks the input~$\bx$ closest to~$\bx^o$ whose fit matches the credences~$\bc$. Its optimal value therefore describes the risk~(perturbation magnitude) of any given credence profile. Note that due to the conflicting natures of fit and risk, the inequality constraints in~\eqref{P:primal} typically hold with equality at the compromise~$\bc^\star$. Immediately, we can write~\eqref{E:compromise} as
\begin{equation}\label{E:preCompromise2}
	r^\star(\bc^\star) - r^\star(\bc^\prime)
		\leq \norm{\bc^\prime}_{\bW^{-1}}^2 - \norm{\bc^\star}_{\bW^{-1}}^2
		\text{.}
\end{equation}
The compromise in Def.~\ref{D:cred} is explicit in~\eqref{E:preCompromise2}: the prediction credibilities of a model/input pair is given, not by the credences induced by the model output in~\eqref{E:fitCred}, but by those that increase confidence in each class at least as much as they increase risk, as measured here by the perturbation magnitude.

The main issue with~\eqref{E:preCompromise2} is that evaluating~$r^\star$ involves solving the optimization problem~\eqref{P:primal}, which may not be computationally tractable. This is not an issue when~\eqref{P:primal} is a convex problem~(e.g., if~$\phi^o$ is convex and~$\ell_k$ is convex and non-decreasing). Yet, typical ML models are non-convex functions of their input, e.g., (C)NNs or RKHS-based methods. To account for cases in which finding the minimum of~\eqref{P:primal} is hard, we consider a local version of Def.~\ref{D:cred} induced by~\eqref{E:preCompromise2}:

\begin{definition}[\emph{Local} credibility profile]\label{D:local_cred}
Let~$\bx^\dagger(\bc)$ be a local minimizer of~\eqref{P:primal} with credences~$\bc$ and consider its value~$r^\dagger(\bc) = \norm{\bx^\dagger(\bc) - \bx^o}^2$. A local credibility profile of the predictions~$\phi^o(\bx^o)$ is a vector~$\bc^\dagger$ that satisfies
\begin{equation}\label{E:compromise2}
	r^\dagger(\bc^\dagger) - \norm{\bx^\prime - \bx^o}^2
		\leq \norm{\bc^\prime}_{\bW^{-1}}^2- \norm{\bc^\dagger}_{\bW^{-1}}^2
		\text{,}
\end{equation}
for all~$\bx^\prime$ in a neighborhood of~$\bx^\dagger(\bc^\dagger)$ that are feasible for~(PI) with credences~$\bc^\prime$.
\end{definition}

While the following results are derived for the local compromise in~\eqref{E:compromise2}, they also hold for Def.~\ref{D:cred} by replacing~${}^\dagger$ with~${}^\star$.

\subsection{Counterfactual evidence}
\label{S:counterfactual}

The following assumptions are used in the sequel:

\begin{assumption}\label{A:diff}
	The loss function~$\ell_k$, $k \in \calK$, and the model~$\phi^o$ are differentiable.
\end{assumption}

\begin{assumption}\label{A:nonconvex}
Let~$\bx^\dagger$ be a local minimizer of~\eqref{P:primal}. There exists a neighborhood~$\calN$ of~$\bx^\dagger$ such that
\begin{equation}\label{E:condition}
\begin{aligned}
	\ell_k\big( \phi^o(\bx) \big) &\geq \ell_k\big( \phi^o(\bx^\dagger) \big)
		+ (\bx - \bx^\dagger)^T \nabla \ell_k\big( \phi^o(\bx^\dagger) \big)
	\\
	{}&- \frac{
		\left[ \ell_k\big( \phi^o(\bx) \big) - \ell_k\big( \phi^o(\bx^\dagger) \big) \right]^2
	}{
		2 \ell_k\big( \phi^o(\bx^\dagger) \big)
	}
		\text{,}
\end{aligned}
\end{equation}
for all~$\bx \in \calN$ and~$k \in \calK$.
\end{assumption}

Assumption~\ref{A:nonconvex} restricts the composite function~$\ell_k\big( \phi^o(\cdot) \big)$ to be non-convex only up to a quadratic in a neighborhood of a local minimum. Additionally, we will only consider credibility profiles for which there exist strictly feasible solutions of~\eqref{P:primal}, i.e., $\bc \in \calC$ for
\begin{multline}\label{E:slater}
	\calC = \big\{ \bc \in \setR^\abs{\calK} \mid \exists \bx \in \setR^p \text{ such that }
	\\
	\ell_k\big( \phi^o(\bx) \big) < -c_k \text{ for all } k \in \calK \big\}
		\text{.}
\end{multline}
Note that~$\calC$ is arbitrarily close to the set of all achievable credences.

To proceed, let us recall the following properties of minimizers:
\begin{theorem}[{KKT conditions,~\citep[Section 5.5.3]{Boyd04c}}]\label{T:kkt}
	Let~$\bx^\dagger$ be a local minimizer of~\eqref{P:primal} for the credences~$\bc \in \calC$. Under Assumption~\ref{A:diff}, there exist~$\blambda^\dagger \in \setR_+^\abs{\calK}$ known as \emph{dual variables} such that
\begin{subequations}\label{E:kkt}
\begin{align}
	2 (\bx^\dagger - \bx^o)
		+ \sum_{k \in \calK} \lambda_k
			\nabla_{\bx} \ell_k\big( \phi^o( \bx^\dagger ) \big) &= 0
		\text{,}
		\label{E:kkt_stationary}
	\\
	\lambda_k^\dagger \left[ \ell_k\big( \phi^o(\bx^\dagger) \big) + c_k \right] &= 0
		\text{,} \quad k \in \calK
		\text{.}
		\label{E:kkt_comp_slack}
\end{align}
\end{subequations}
\end{theorem}

If~$\ell_k\big( \phi^o(\cdot) \big)$ is convex, \ref{E:kkt} are both necessary and sufficient for global optimality. Additionally, the function~$r^\dagger$~(or in this case, $r^\star$) is differentiable and its derivative with respect to~$\bc$ is given by the dual variables. In fact, $\blambda^\dagger(\bc)$ then quantifies the change in the risk~$r^\dagger(\bc)$ if the credences increase or decrease by a small amount. This sensitivity interpretation is a classical result from the convex optimization literature~\citep[Section 5.6.2]{Boyd04c}.

In general, however, $\ell_k\big( \phi^o(\cdot) \big)$ is non-convex. Still, Theorem~\ref{T:kkt} allows us to obtain a fixed-point condition that turns credences into credibilities.

\begin{theorem}\label{T:counterfactual}

Take~$\bc \in \calC$ and let~$\bx^\dagger(\bc)$ be a local minimizer of~\eqref{P:primal} with value~$r^\dagger(\bc)$ associated with the dual variables~$\blambda^\dagger(\bc)$. Under Assumptions~\ref{A:diff} and~\ref{A:nonconvex}, the local credibility profile~$\bc^\dagger$ from Def.~\ref{D:local_cred} exists and satisfies
\begin{equation}\label{E:counterfactual}
	\bc = - \frac{1}{2} \bW \blambda^\dagger(\bc)
	\Rightarrow
	\bc \text{ satisfies~\eqref{E:compromise2}, i.e., } \bc = \bc^\dagger
		\text{.}
\end{equation}
\end{theorem}

\begin{proof}
See appendix~\ref{X:counterfactual}.
\end{proof}

Theorem~\ref{T:counterfactual} shows that Theorem~\ref{T:kkt} can be used to certify a local credibility measure~$\bc^\dagger$ without repeatedly solving~\eqref{P:primal}. Explicitly, if~\eqref{E:kkt} hold with~$\blambda^\dagger = -2 \bW^{-1} \bc$, then~$\big( \bx^\dagger(\bc), \bc \big)$ satisfy~\eqref{E:compromise2}. In other words, the dual variables associated to solutions of~\eqref{P:primal} provide counterfactuals of the form ``if the credence~$c_k$ assigned to class~$k$ had been~$\lambda_k^\dagger$ for all~$k \in \calK$, then~$\bc$ would have been a local credibility profile.'' Considering the sensitivity interpretation of~$\blambda^\dagger$ in the convex case, the compromise~\eqref{E:compromise2} also causes the credence~$c_k$ for classes in which~$\bx^o$ is harder to fit to decrease~(become more negative) in order to manage risk.

In the sequel, we leverage the results from Theorems~\ref{T:kkt} and~\ref{T:counterfactual} to first show that credibility as defined in Def.~\ref{D:local_cred} has a Bayesian interpretation as the MAP of a specific failure model~(Section~\ref{S:bayesian}) and then put forward an algorithm to efficiently compute the credibility profile~$\bc^\dagger$ without repeatedly solving~\eqref{P:primal} for different credence profiles.

\section{A Bayesian formulation of credibility}
\label{S:bayesian}

An interesting consequence of Theorem~\ref{T:counterfactual}, is that the fit-risk trade-off credibility definitions in Def.~\ref{D:cred} and~\ref{D:local_cred} can also be treated as a Bayesian inference problem. This problem formalizes the intuition that formulating~\eqref{E:compromise} [and more generally~\eqref{E:compromise2}] in terms of Euclidian norms is equivalent to model the uncertainties related to the input and credibility profiles as Gaussian random variables~(RVs).

Consider the likelihood
\begin{subequations}\label{E:bayesianCred}
\begin{align}\label{E:likelihoodCred}
	\Pr\left( \bx \mid \bc \right) &=
		\calN\left( \bx \mid \bx^o, (2t)^{-1} \bI \right)
	\notag\\
	&{}\times \prod_{c_k \neq 0}
			\text{SE}\left[ \ell_k\big( \phi^o(\bx) \big) \mid c_k, \frac{2 c_k}{w_k} t \right]
		\text{,}
\end{align}
defined for a parameter~$t > 0$, where~$\calN\left( \bz \mid \bzb, \bSigma \right)$ represents the density of a normal random vector~$\bz$ with mean~$\bzb$ and covariance~$\bSigma$ and~$\text{SE}\left( z \mid \bar{z}, \eta \right)$, the density of an exponentially distributed RV~$z$ with rate~$\eta$ shifted to~$\bar{z}$. This likelihood represents our belief that, though~$\bx^o$ is representative of the model input, it may be corrupted. This uncertainty is described by the first term of~\eqref{E:likelihoodCred}.

The other terms account for the failure of~$\bx$ to meet the credence~$\bc$ in the constraints of~\eqref{P:primal}. If an input~$\bx$ violates any of the credences, i.e., $\ell_k\big( \phi^o(\bx) \big) > -c_k$, then its likelihood is penalized independently of how close it is to the mean~$\bx^o$. This penalty follows a constant failure rate that depends on the credence~$c_k$ itself. Hence, the probability of the violation increasing by~$\epsilon$ does not depend on how much the constraint has already been violated. Explicitly,
\begin{align*}
	\Pr\left[ \ell_k\big( \phi^o(\bx) \big) > -c_k + z + \epsilon \mid
		\ell_k\big( \phi^o(\bx) \big) > -c_k + z \right]
	\\
	{}= \Pr\left[ \ell_k\big( \phi^o(\bx) \big) > -c_k + \epsilon \right]
		\text{.}
\end{align*}

To obtain the joint distribution of~$(\bx,\bc)$, we use a normal prior on the credence values, namely
\begin{equation}\label{E:priorCred}
	\Pr\left( \bc \right) = \calN\left( \bc \mid \zeros, (2t)^{-1} \bW \right)
		\text{.}
\end{equation}
\end{subequations}
Observe from~\eqref{E:priorCred} that~$\bW$ in~\eqref{E:compromise}~[and~\eqref{E:compromise2}] can be interpreted both as a weighting matrix that modifies the geometry of the credence space or as a measure of uncertainty over the prediction credibilities. It can therefore be used to incorporate prior information, e.g., on the relative frequency of classes. The following proposition characterizes a local maximum of the joint distribution~$\Pr(\bx,\bc)$:

\begin{proposition}\label{T:map}

Let the pair~$(\bx^\dagger,\bc^\dagger)$ satisfy the local credibility compromise~\eqref{E:compromise2}. Then, there exists~$T > 0$ such that it is a local maximum of the joint distribution~$\Pr\left( \bx, \bc \right) = \Pr\left( \bx \mid \bc \right) \Pr\left( \bc \right)$ defined by~\eqref{E:bayesianCred} for~$t \geq T$.

\end{proposition}

\begin{proof}
See appendix~\ref{X:map}.
\end{proof}

Proposition~\ref{T:map} shows that the local credibility profiles in~\eqref{E:compromise2} are asymptotic local maxima of the probabilistic model~\eqref{E:bayesianCred} as the variance of its components decreases, i.e., as~$t$ increases and the joint probability distribution becomes more concentrated. In fact, every critical point of~$\Pr\left( \bx, \bc \right)$ satisfies the KKT conditions of Theorem~\ref{T:kkt} and the counterfactual condition of Theorem~\ref{T:counterfactual}. Recall that if~$\phi^o$ is convex, then the unique mode of this joint distribution~$(\bx^\star,\bc^\star)$ provides a credibility profile according to Def.~\eqref{D:cred}. Hence, though motivated as deterministic fit-risk trade-offs, the credibility metrics proposed in this work can also be viewed as MAP estimates.

In the next section, we conclude our analysis of the credibility profile in Def.~\ref{D:local_cred} by showing how it can be computed efficiently from~\eqref{P:primal} at no additional complexity cost. We do so by modifying the Arrow-Hurwicz algorithm~\citep{Arrow58s} using Theorem~\ref{T:counterfactual}.

\section{A modified Arrow-Hurwicz algorithms}
\label{S:algorithm}

\begin{algorithm}[t]
\centering
\caption{Counterfactual optimization algorithm}
	\label{L:saddle}
	\setlength{\baselineskip}{1.2\baselineskip}
	\footnotesize
	\begin{algorithmic}
		\STATE Let $\bx^{(0)} = \bx$ and~$\blambda^{(0)} = \ones$.

		\FOR{$t = 1,2,\dots$}

			\STATE $\displaystyle
				g_{\bx}^{(t)} = 2 \big( \bx^{(t)} - \bx^o \big)
					+ \sum_{k \in \calK} \lambda_k^{(t)}
						\nabla_{\bx} \ell_k\left( \phi^o\big( \bx^{(t)} \big) \right)$

			\STATE $\displaystyle
				\bx^{(t+1)} = \bx^{(t)} - \eta_{\bx} g_{\bx}^{(t)}$

			\STATE $\displaystyle
				\lambda_k^{(t+1)} = \left[ \lambda_k^{(t)}
				+ \eta_{\blambda} \left[ \ell_k\left( \phi^o\big( \bx^{(t)} \big) \right)
					- \lambda_k^{(t)} \right]
				\vphantom{\sum_1^0}\right]_+$

		\ENDFOR
	\end{algorithmic}
\end{algorithm}

Theorems~\ref{T:kkt} and~\ref{T:counterfactual} suggest a way to exploit the information in the dual variables to solve~\eqref{P:primal} directly for the~(local) credibility~$\bc^\dagger$ without testing multiple credence profiles. Indeed, start by considering that the credences~$\bc$ are fixed and define the Lagrangian associated with~\eqref{P:primal} as
\begin{equation}\label{E:lagrangian}
	\calL(\bx, \blambda, \bc) = \norm{\bx - \bx^o}^2 + \sum_{k \in \calK} \lambda_k
		\left[ \ell_k\big( \phi^o(\bx) \big) + c_k \right]
		\text{.}
\end{equation}
Observe that the KKT necessary condition~\eqref{E:kkt} for~$\bx^\dagger$ to be a local minimizer can be written in terms of~\eqref{E:lagrangian} as~$\nabla_{\bx} \calL(\bx^\dagger, \blambda^\dagger, \bc) = 0$ and~$\lambda_k^\dagger \left[ \nabla_{\blambda} \calL(\bx^\dagger, \blambda^\dagger, \bc) \right]_k = 0$ for~$k \in \calK$, where~$[\bz]_k$ indicates the~$k$-th element of the vector~$\bz$. The classical Arrow-Hurwicz algorithm~\citep{Arrow58s} is a procedure inspired by these relations that seeks a KKT point by alternating between the updates
\begin{subequations}\label{E:arrow}
\begin{align}
	\bx^+ &= \bx
		- \eta_{\bx} \nabla_{\bx} \calL\left( \bx, \blambda, \bc \right)
	\label{E:arrow_primal}\\\notag
	{}&= \bx - \eta_{\bx} \Bigg[ 2(\bx - \bx^o) + \sum_{k \in \calK} \lambda_k
		\nabla_{\bx} \ell_k\big( \phi^o(\bx) \big) \Bigg]
		\text{,}
	\\
	\lambda_k^+ &= \Bigg[ \lambda_k + \eta_{\blambda}
		\big[
			\nabla_{\blambda} \calL\left( \bx, \blambda, \bc \right)
		\big]_k
	\Bigg]_+
	\label{E:arrow_dual}\\\notag
	{}&= \Bigg[ \lambda_k
		+ \eta_{\blambda} \bigg( \ell_k\big( \phi^o(\bx) \big) + c_k \bigg)
	\Bigg]_+
		\text{,}
\end{align}
\end{subequations}
where~$\eta_{\bx},\eta_{\blambda} > 0$ are step sizes and~$[\bz]_+ = \max(\bz,\zeros)$ denotes the projection onto the non-negative orthant of~$\setR^\abs{\calK}$.

To understand the intuition behind this algorithm, note that~\eqref{E:arrow_primal} updates~$\bx$ by descending along a weighted combination of gradients of the objective and the constraints so as to reduce the value of all functions. The weight of each constraint is given by its respective dual variable~$\lambda_k$. If the $k$-th constraint is satisfied, then~$\ell_k\big( \phi^o(\bx) \big) + c_k \leq 0$ and its influence on the update of~$\bx$ is decreased by~\eqref{E:arrow_dual} until it vanishes. On the other hand, if the constraint is violated, then~$\ell_k\big( \phi^o(\bx) \big) + c_k > 0$ and the value of~$\lambda_k$ increases. The relative strength of each gradient in the update~\eqref{E:arrow_primal} is therefore related to the history of violation of each constraint.

The main drawback of~\eqref{E:arrow} is that it seeks a KKT point of~\eqref{P:primal} for a given, fixed credence profile~$\bc$ while the credibility profile~$\bc^\dagger$ from~\eqref{E:compromise2} is not known \emph{a priori}. To overcome this issue, we can use the counterfactual result~\eqref{E:counterfactual} in~\eqref{E:arrow_dual} to obtain
\begin{equation}\label{E:arrow_dual_mod}
	\lambda_k^+ = \Bigg[
		\lambda_k + \eta_{\blambda} \bigg( \ell_k\big( \phi^o(\bx) \big)
					- \frac{1}{2} w_k \lambda_k \bigg) \vphantom{\sum}
	\Bigg]_+
		\text{,}
\end{equation}
The complete counterfactual optimization procedure is collected in Algorithm~\ref{L:saddle}. If the dynamics of Algorithm~\ref{L:saddle} converge to~$(\bx_\infty,\blambda_\infty)$, then~$\bx_\infty$ is a local minimizer of~\eqref{P:primal} for credences~$\bc = -\frac{1}{2} \bW \blambda_\infty$, which is a local credibility profile~$\bc^\dagger$ according to Theorem~\ref{T:counterfactual} .

It is worth noting that in the general, non-convex case, Algorithm~\ref{L:saddle} need not converge. This will happen if the gradient descent procedure in~\eqref{E:arrow_primal} is unable to find a~$\bx$ that fits the credences imposed by~$\blambda$. For rich model classes, this is less likely to happen and there is considerable empirical evidence from the adversarial examples literature that gradient methods such as~\eqref{E:arrow_primal} do converge~\citep{Szegedy14i, Goodfellow14e, Madry18t}. This is also what we observed in our numerical experiments, during which we found no instance in which Algorithm~\ref{L:saddle} diverged~(Section~\ref{S:sims}). When~$\ell_k$ is convex and non-decreasing and~$\phi^o$ is convex, then Algorithm~\ref{L:saddle} can be shown to converge to the global optima~$(\bx^\star,\blambda^\star,\bc^\star)$ of~\eqref{P:primal} through classical arguments, as in~\citep{Cherukuri16a, Nagurney12p}. Details of this result are beyond the scope of this paper.

\section{Numerical experiments}
\label{S:sims}

\begin{figure}[tb]
\begin{minipage}[c]{0.49\columnwidth}
\centering
\includegraphics[width=\columnwidth]{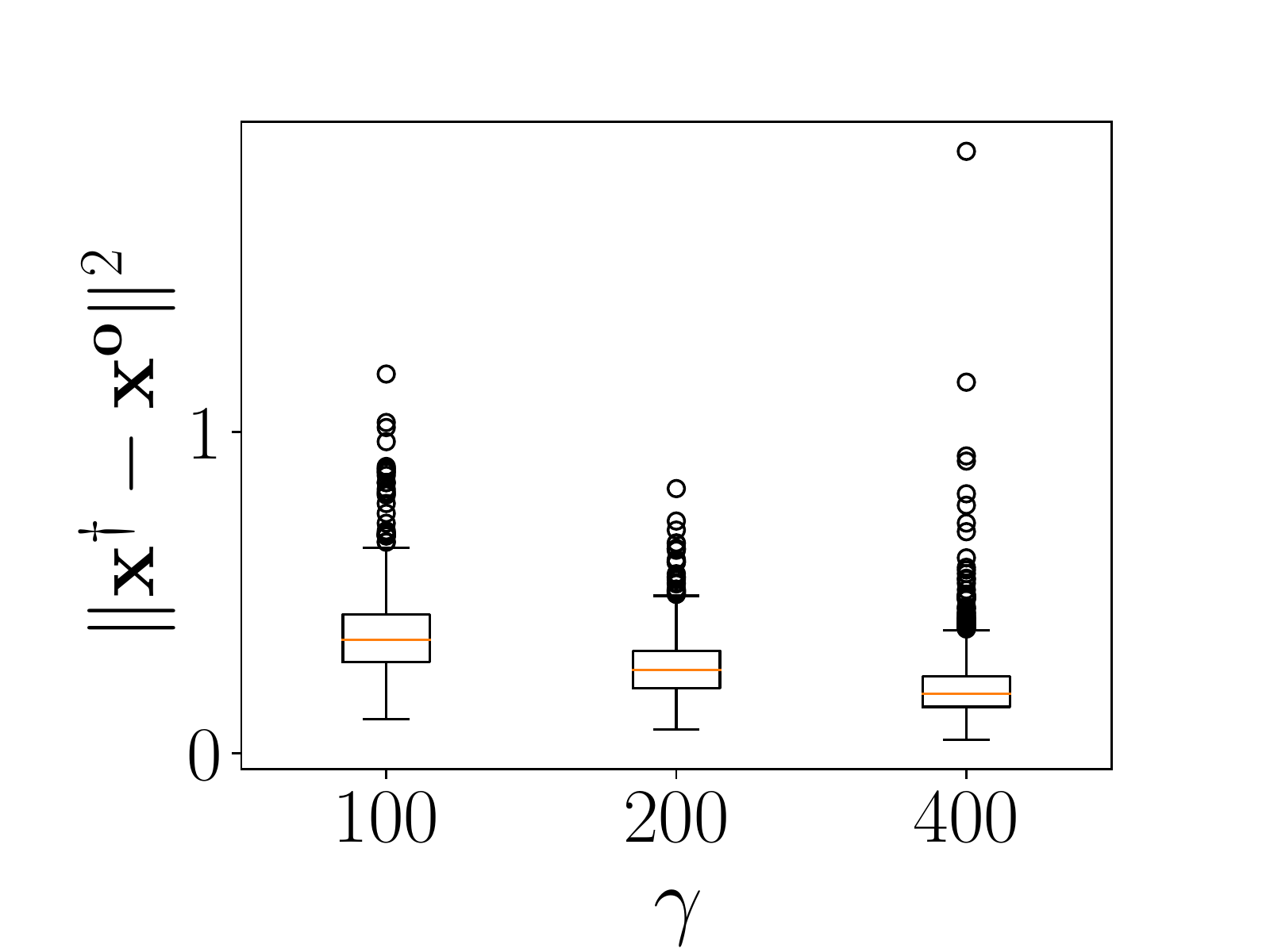}

{\small (a)}
\end{minipage}
\hfill
\begin{minipage}[c]{0.49\columnwidth}
\centering
\includegraphics[width=\columnwidth]{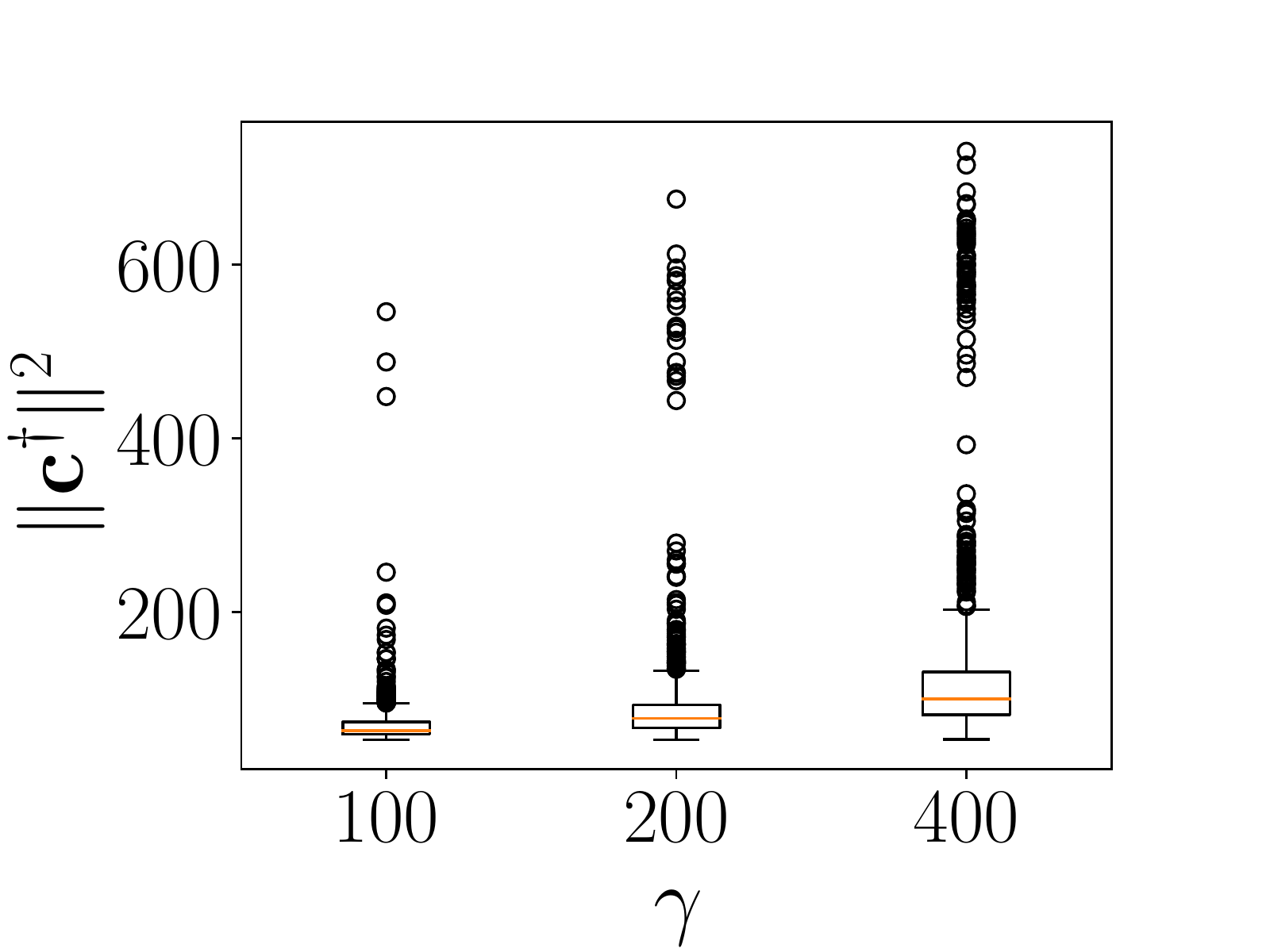}

{\small (b)}
\end{minipage}

\caption{(a)~Perturbation magnitude and (b)~Overall credibility}
	\label{F:compromise}
\end{figure}

\begin{figure}[tb]
\begin{minipage}[c]{0.49\columnwidth}
\centering
\includegraphics[width=\columnwidth]{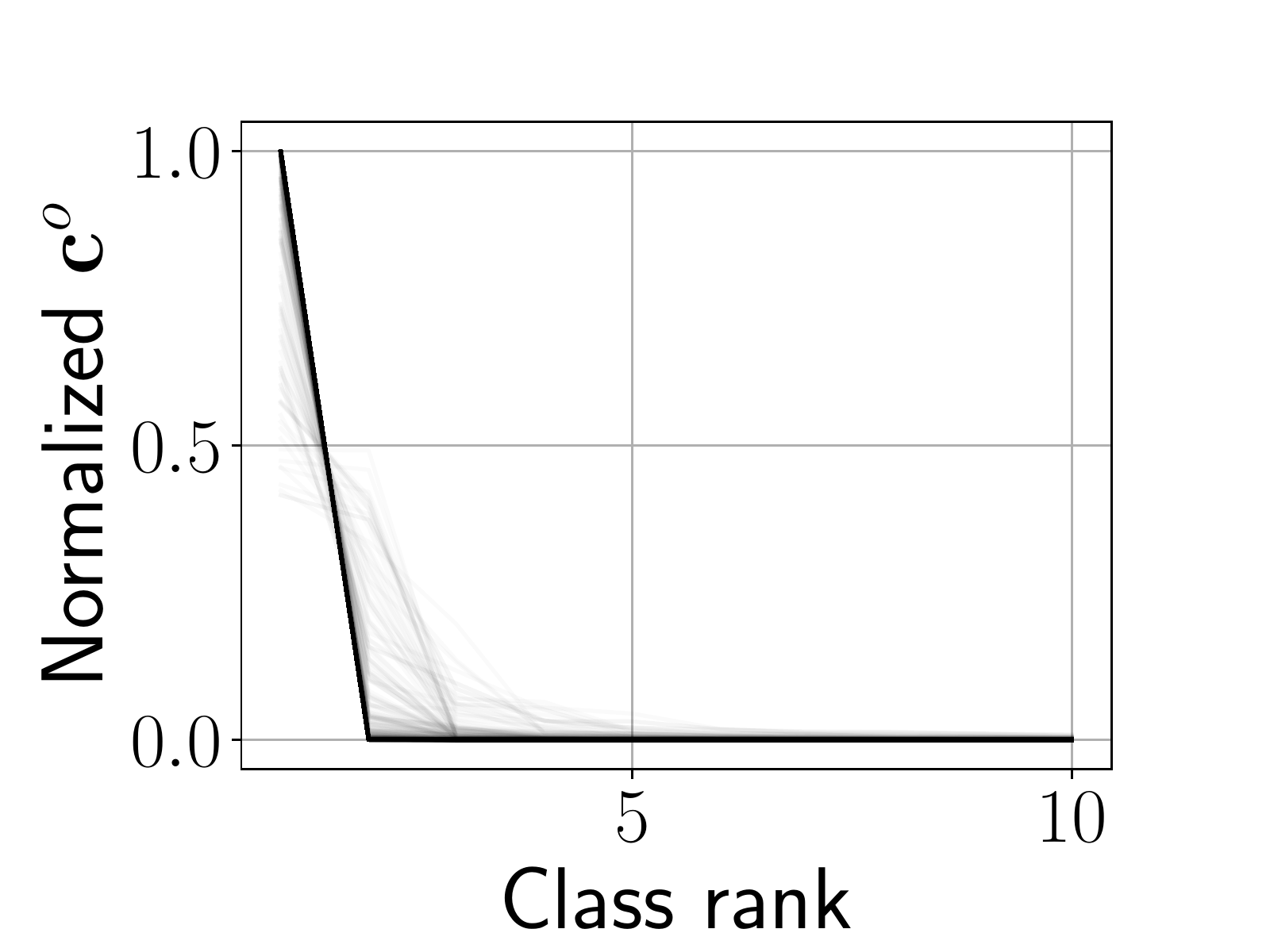}

{\small (a) ResNet18}
\end{minipage}
\hfill
\begin{minipage}[c]{0.49\columnwidth}
\centering
\includegraphics[width=\columnwidth]{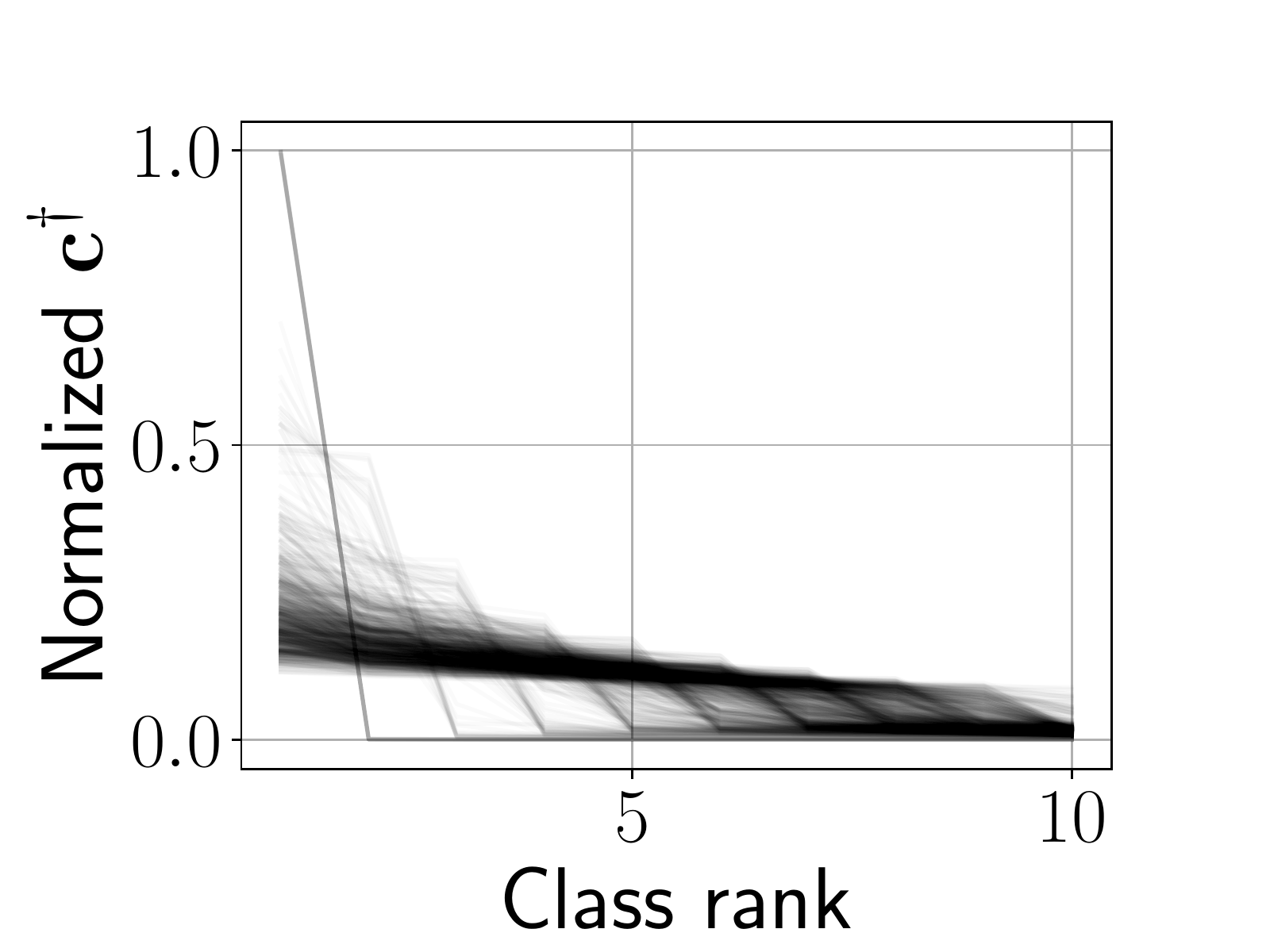}

{\small (b) ResNet18}
\end{minipage}

\begin{minipage}[c]{0.49\columnwidth}
\centering
\includegraphics[width=\columnwidth]{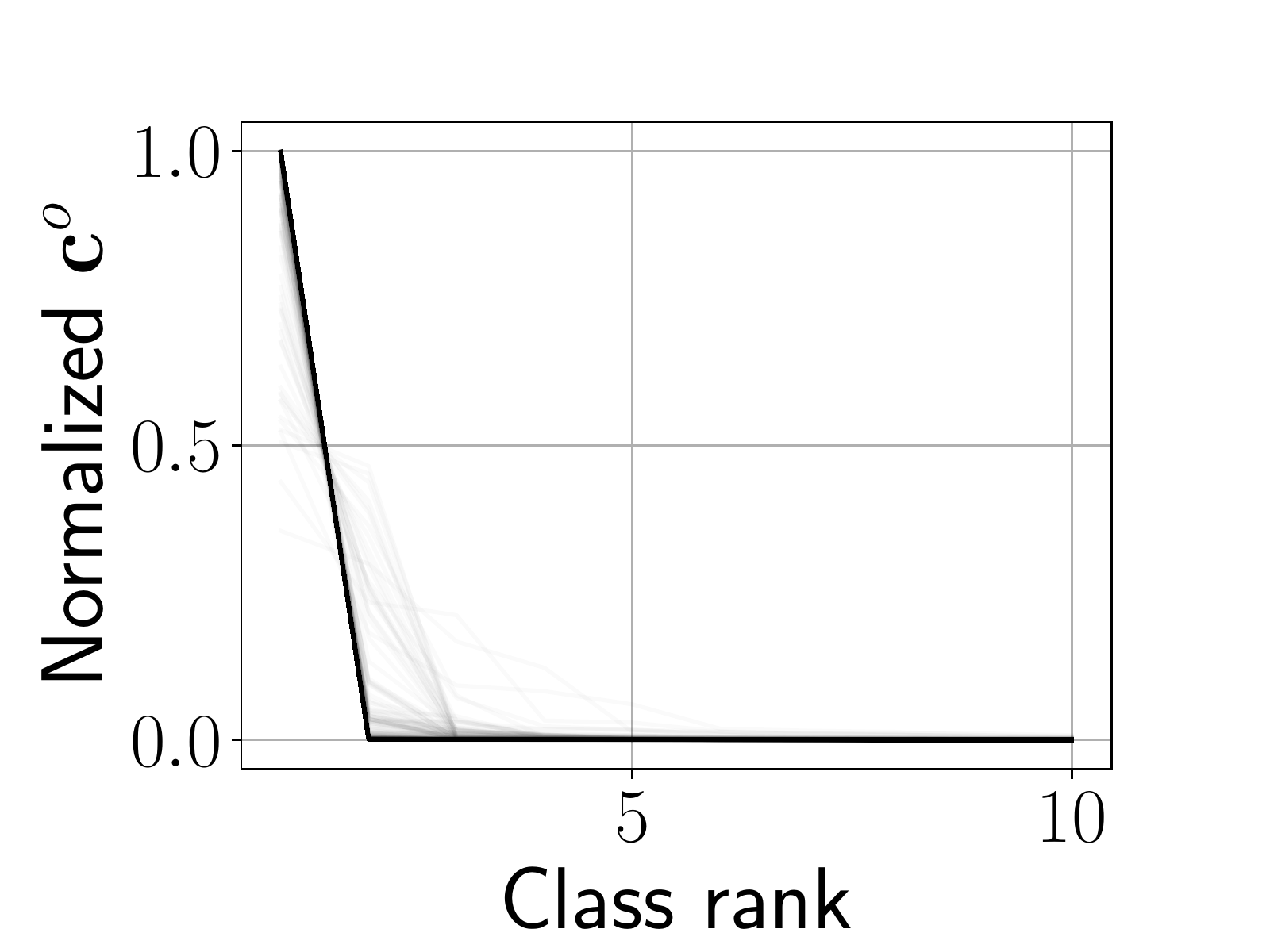}

{\small (c) DenseNet169}
\end{minipage}
\hfill
\begin{minipage}[c]{0.49\columnwidth}
\centering
\includegraphics[width=\columnwidth]{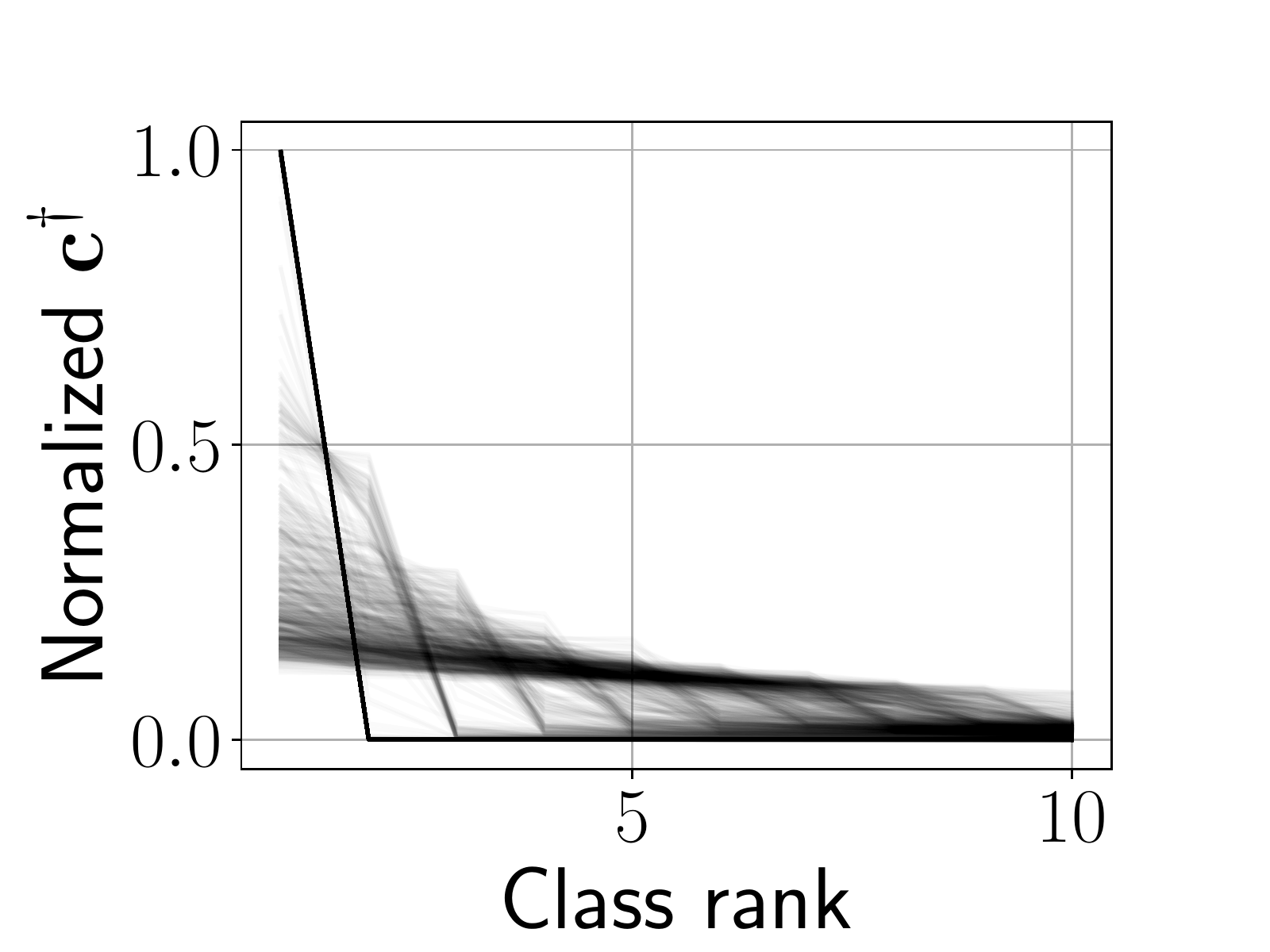}

{\small (d) DenseNet169}
\end{minipage}

\caption{Credibility profiles for ResNet18 and DenseNet169}
	\label{F:res_vs_dense}
\end{figure}

To showcase the properties and uses of this trade-off approach to credibility, we use two CNN architectures, a ResNet18~\citep{He16d} and a DenseNet169~\citep{Huang17d}, trained to classify images from the CIFAR-10 dataset. The models were trained over $600$~epochs in mini-batches of $250$~samples using Adam with the default parameters from~\citep{Kingma17a} with a weight decay parameter of~$10^{-3}$. Random cropping and flipping was used to augment data during training. The final classifiers achieve validation accuracies of~$93\%$ for the ResNet18 and~$94\%$ for the DenseNet169. Throughout this section, the loss function~$\ell_k$ is the cross-entropy loss and all experiments were performed over a~$1000$ images random sample from the CIFAR-10 test set.

We start by illustrating the effect of the trade-off between fit and risk~(perturbation magnitude). To do so, we compute the credibility~$\bc^\dagger$ for~$\bW = \gamma \bI$ with~$\gamma \in \{100, 200, 400\}$ using the ResNet18 CNN as~$\phi^o$~(Figure~\ref{F:compromise}). Observe that as~$\gamma$ increases, the compromise~\eqref{E:compromise} becomes more risk averse. Hence, the magnitude of the perturbations decrease~(Figure~\ref{F:compromise}a). As this happens, the norm of the credibility profiles~$\norm{\bc^\dagger}^2$ increases~(Figure~\ref{F:compromise}b). Yet, note that the trade-off in~\eqref{E:preCompromise} continues to hold for~$\gamma^{-1} \norm{\bc^\dagger}^2$. In the sequel, we proceed using~$\gamma = 200$.

Given the perturbation stability interpretation of~$\bc^\dagger$, it can be used to analyze the robustness of different classifiers. By studying the relation between the ordering induced on the classes~$\calK$ and the values of~$\bc^o$ and~$\bc^\dagger$~(Fig.~\ref{F:res_vs_dense}), we can evaluate how input perturbations alter the model outputs. Recall that the perturbation magnitude is adapted to both the model and the input. Hence, while perturbation levels might be different, they are the limit of the compromise in~\eqref{E:compromise2}, so that any larger perturbation would not incur in a significantly larger change in the model output. Rather than assessing robustness to a set of perturbations, this analysis evaluates the failure modes of the model, i.e., when and how models fail.

\begin{figure}[tb]
\centering
\includegraphics[width=\columnwidth]{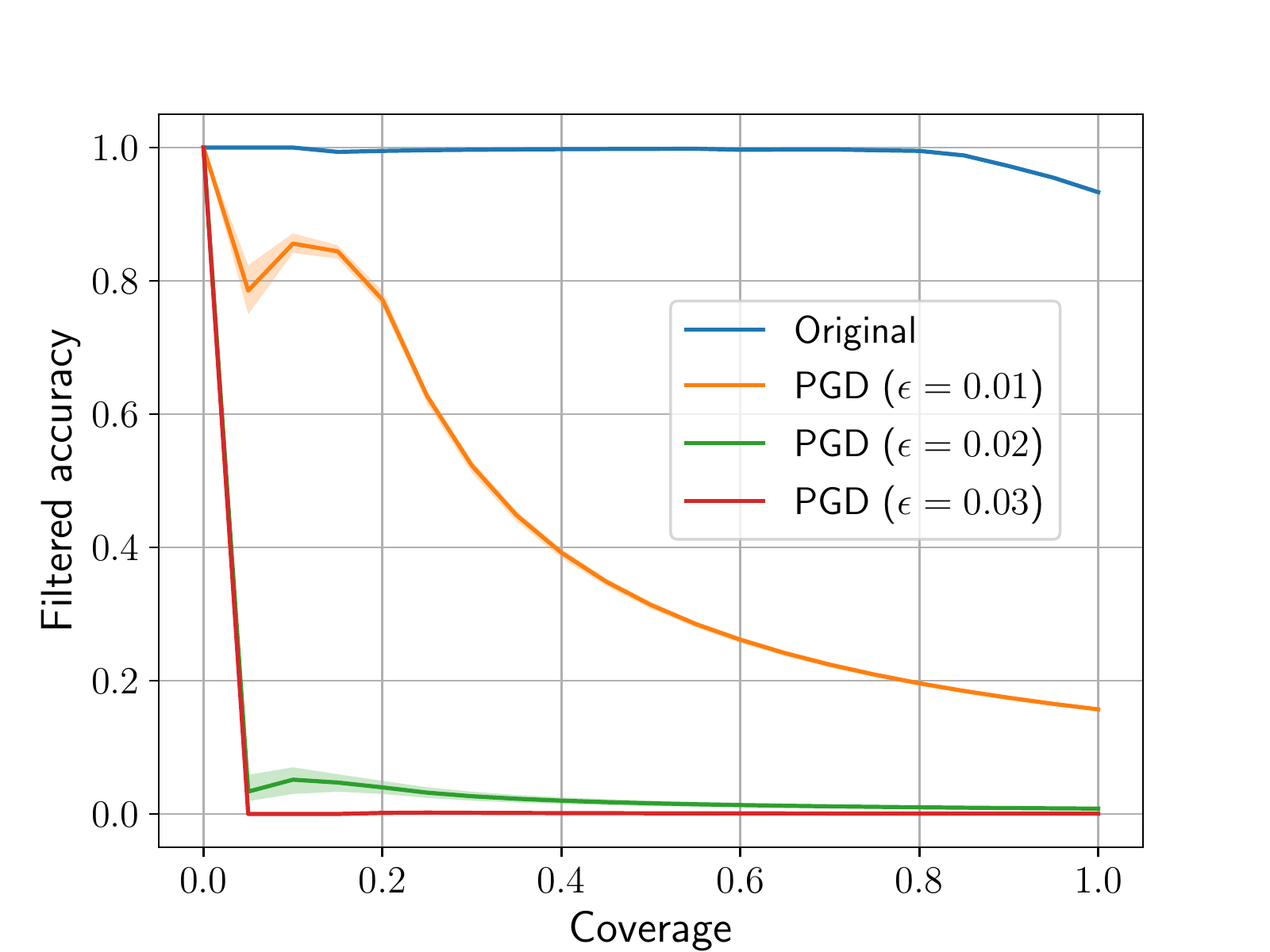}
\caption{Performance of softmax-based filtered classifiers. Shades display best and worst case results over~$10$ restarts.}
\label{F:filtering_softmax}
\end{figure}

In Figure~\ref{F:res_vs_dense}, we display the values of~$\bc^o$ and~$\bc^\dagger$ sorted in decreasing order for the pretrained ResNet18 and DenseNet169 CNNs. To help visualization, the credibility profiles were normalized so as to lie in the unit interval. Notice that, in contrast to the DenseNet, the perturbed profiles of the Resnet~(Figure~\ref{F:res_vs_dense}b) are considerably different than the original ones~(Figure~\ref{F:res_vs_dense}a). This implies that the ResNet output can be considerably modified by perturbations that are small compared to their effect~(as in Def.~\ref{D:local_cred}). On the other hand, note that the DenseNet retains a good amount of certainty on some samples even after perturbation~(Figure~\ref{F:res_vs_dense}d). While in~$\bc^o$ this certainty can be artificial~(e.g., due to miscalibration), the fact that~$\bc^\dagger \approx \bc^o$ means that modifying the classifier decision~(whether correct or not) would require too large perturbations to be warranted. In this case, the DenseNet model displays signs of stability that the ResNet does not. It is worth pointing out that both architecture have a similar numbers of parameters. We proceed with the experiments using the ResNet.

Another use of credibility is in the context of filtering. In this case, credibility is used to assess whether the classifier is confident enough to make a prediction. If not, then it should abstain from doing so. The performance of such filtered classifiers is evaluated in terms of its coverage~(how many samples the model chooses to classify) and filtered accuracy~(accuracy over the classified samples). Ideally, we would like a model that never abstains~($100\%$ coverage) and is always accurate. Yet, these two quantities are typically conflicting: filtered accuracy can often be improved by reducing coverage. The issue is then whether this can be done smoothly. Suppose that the model decides to classify a sample if its second largest credibility is less than the largest one by a factor of at least~$1-\alpha$, where~$\alpha$ is chosen to achieve specific coverage levels. Figures~\ref{F:filtering_softmax} and~\ref{F:filtering_cred} compare the results of directly using the softmax output of the model and the credibility profile~$\bc^\dagger$.

\begin{figure}[tb]
\centering
\includegraphics[width=\columnwidth]{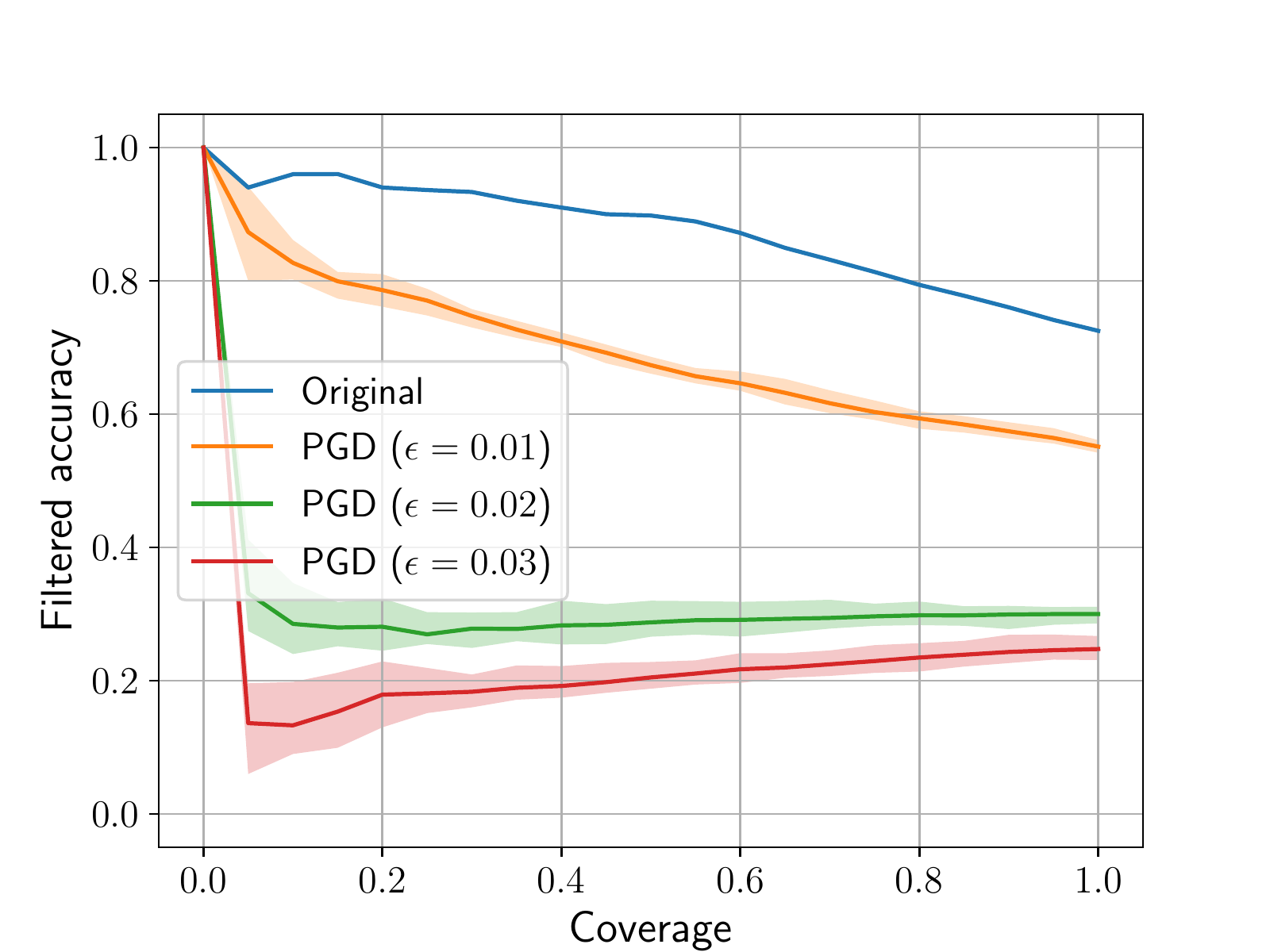}
\caption{Performance of credibility-based filtered classifiers. Shades display best and worst case results over~$10$ restarts.}
\label{F:filtering_cred}
\end{figure}

\begin{figure}[tb]
\centering
\includegraphics[width=\columnwidth]{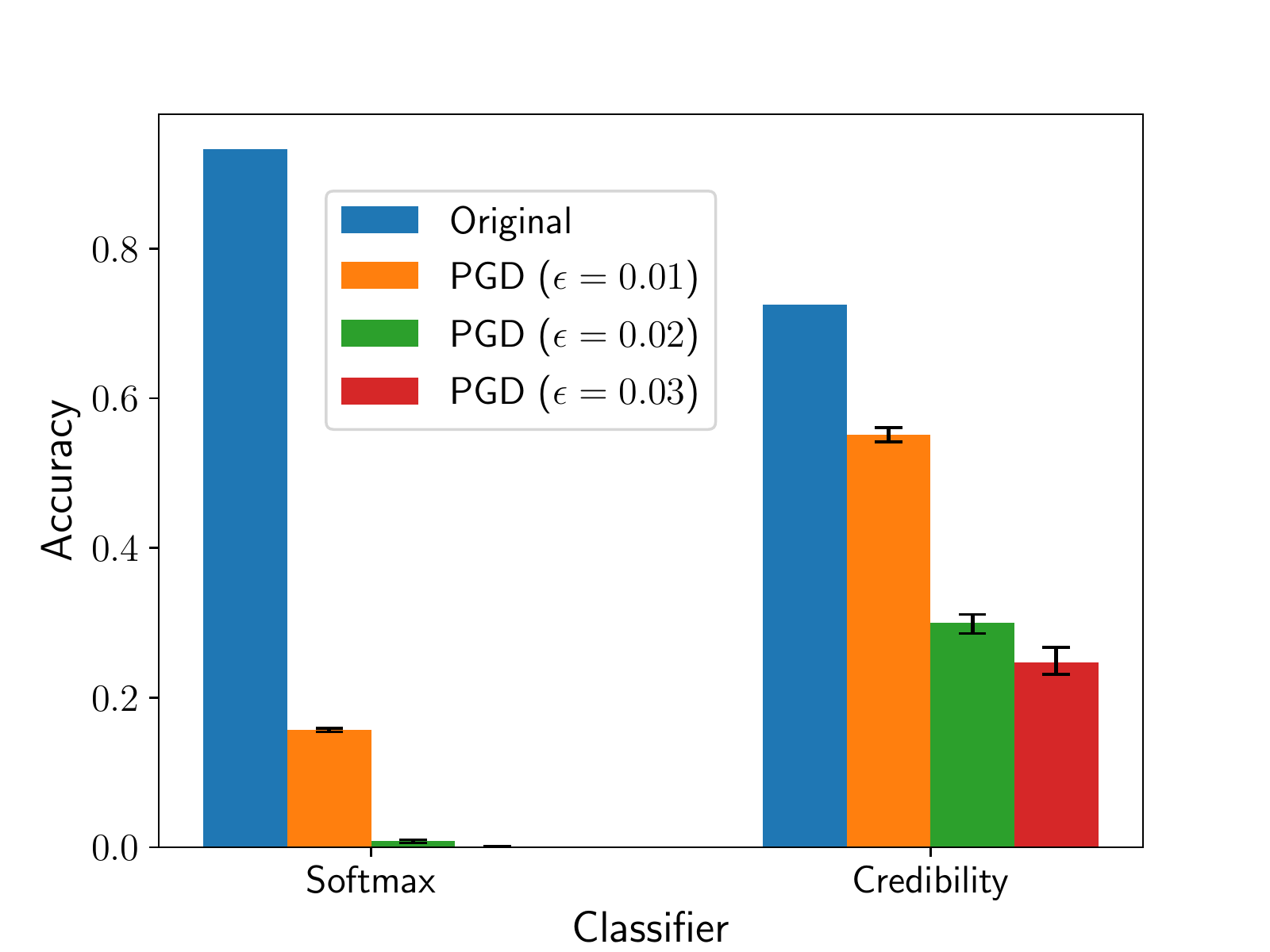}
\caption{Accuracy of softmax- and credibility-based classifiers}
\label{F:accuracy}
\end{figure}

Notice that when classifying every sample, the classifier based on~$\bc^\dagger$ has lower accuracy~(see also Figure~\ref{F:accuracy}). However, it is able to discriminate between certain and uncertain predictions in a way that the output of the model cannot. This is straightforward given that the model output typically has a single strong component~(see Figure~\ref{F:res_vs_dense}a). This shows that the pretrained model is, in many cases, overconfident about its predictions, a phenomenon related to the issue of miscalibration in NNs~\citep{Guo17o}. This overconfidence becomes clear when the data is attacked using adversarial noise. To illustrate this effect, we perturb the test images using the projected gradient descent~(PGD) attack for maximum perturbations~$\epsilon \in \{0.01,0.02,0.03\}$~\citep{Madry18t}. It is worth noting that for the largest~$\epsilon$, the perturbations are noticeable in the image. PGD was run for $100$~iterations, with step size~$0.01$~($0.005$ for~$\epsilon = 0.01$), and the figures show the range of results obtained over~$10$ restarts.

The accuracy of the classifiers decreases considerably after the PGD attack~(Figure~\ref{F:accuracy}). However, while the performance of the softmax classifier degrades by almost~$78\%$, the credibility-based ones drops by~$17\%$. When the magnitude of the perturbation reaches~$\epsilon = 0.03$, the output of the model was only able to correctly classify a single image in half of the restarts. Using credibility, however, accuracy remained above~$23\%$. It is worth noting that this result is achieved without any modification to the model, including retraining or preprocessing. What is more, filtering does not improve upon these results due to the miscalibration issue. Indeed, even for~$\epsilon = 0.01$, the softmax-based filtered classifier must give up on over~$70\%$ of its coverage to recover the performance that the credibility-based filtered classifier has classifying every sample. These experiments illustrate the trade-off between robustness and nominal accuracy noted in~\citep{Tsipras19r}.

\section{Conclusion}

This work introduced a credibility measure based on a compromise between how much the model fit can be improved by perturbing its input and how much this perturbation modifies the original input~(risk). By formulating this problem in the language of constrained optimization, it showed that this trade-off can be determined counterfactually, without testing multiple perturbations. Leveraging these results, it put forward a practical method to assign credibilities that (i)~can be computed for any~(possibly non-convex) differentiable model, from RKHS-based solutions to any (C)NN architecture, (ii)~can be obtained for models that have already been trained, (iii)~does not rely on any form of training data, and (iv)~has formal guarantees. Future works include devising local sensitivity results for non-convex optimization programs, analyze the model perturbation problem, and apply the counterfactual compromise result to different learning problems, such as outlier detection. The Bayesian formulation from Section~\ref{S:bayesian} also suggests an independent line of work linking constrained optimization and MAP estimates.

\bibliographystyle{icml2020}
\bibliography{aux_files/IEEEabrv,aux_files/af,aux_files/bayes,aux_files/control,aux_files/gsp,aux_files/math,aux_files/ml,aux_files/rkhs,aux_files/rl,aux_files/sp,aux_files/stat}

\begin{thebibliography}{42}
\providecommand{\natexlab}[1]{#1}
\providecommand{\url}[1]{\texttt{#1}}
\expandafter\ifx\csname urlstyle\endcsname\relax
  \providecommand{\doi}[1]{doi: #1}\else
  \providecommand{\doi}{doi: \begingroup \urlstyle{rm}\Url}\fi

\bibitem[Aggarwal(2015)]{Aggarwal15o}
Aggarwal, C.~C.
\newblock Outlier analysis.
\newblock In \emph{Data mining}, pp.\  237--263. Springer, 2015.

\bibitem[Arrow et~al.(1958)Arrow, Hurwicz, and Uzawa]{Arrow58s}
Arrow, K., Hurwicz, L., and Uzawa, H.
\newblock \emph{Studies in linear and non-linear programming}.
\newblock Stanford University Press, 1958.

\bibitem[Blundell et~al.(2015)Blundell, Cornebise, Kavukcuoglu, and
  Wierstra]{Blundell15w}
Blundell, C., Cornebise, J., Kavukcuoglu, K., and Wierstra, D.
\newblock Weight uncertainty in neural networks.
\newblock \emph{arXiv preprint arXiv:1505.05424}, 2015.

\bibitem[Bonnans \& Shapiro(2000)Bonnans and Shapiro]{Bonnans00p}
Bonnans, J. and Shapiro, A.
\newblock \emph{Perturbation Analysis of Optimization Problems}.
\newblock Springer, 2000.

\bibitem[Boyd \& Vandenberghe(2004)Boyd and Vandenberghe]{Boyd04c}
Boyd, S. and Vandenberghe, L.
\newblock \emph{Convex optimization}.
\newblock Cambridge University Press, 2004.

\bibitem[Chandola et~al.(2009)Chandola, Banerjee, and Kumar]{Chandola09a}
Chandola, V., Banerjee, A., and Kumar, V.
\newblock Anomaly detection: A survey.
\newblock \emph{ACM computing surveys (CSUR)}, 41\penalty0 (3):\penalty0 15,
  2009.

\bibitem[Cherukuri et~al.(2016)Cherukuri, Mallada, and
  Cort\'{e}s]{Cherukuri16a}
Cherukuri, A., Mallada, E., and Cort\'{e}s, J.
\newblock Asymptotic convergence of constrained primal--dual dynamics.
\newblock \emph{Systems \& Control Letters}, 87:\penalty0 10--15, 2016.

\bibitem[Cortes et~al.(2016)Cortes, DeSalvo, and Mohri]{Cortes16b}
Cortes, C., DeSalvo, G., and Mohri, M.
\newblock Boosting with abstention.
\newblock In Lee, D.~D., Sugiyama, M., Luxburg, U.~V., Guyon, I., and Garnett,
  R. (eds.), \emph{Advances in Neural Information Processing Systems}, pp.\
  1660--1668. 2016.

\bibitem[Dhillon et~al.(2018)Dhillon, Azizzadenesheli, Bernstein, Kossaifi,
  Khanna, Lipton, and Anandkumar]{Dhillon18s}
Dhillon, G.~S., Azizzadenesheli, K., Bernstein, J.~D., Kossaifi, J., Khanna,
  A., Lipton, Z.~C., and Anandkumar, A.
\newblock Stochastic activation pruning for robust adversarial defense.
\newblock In \emph{International Conference on Learning Representations}, 2018.

\bibitem[Dietterich(2000)]{Dietterich00e}
Dietterich, T.~G.
\newblock Ensemble methods in machine learning.
\newblock In \emph{MULTIPLE CLASSIFIER SYSTEMS, LBCS-1857}, pp.\  1--15.
  Springer, 2000.

\bibitem[Efron \& Tibshirani(1994)Efron and Tibshirani]{Efron94a}
Efron, B. and Tibshirani, R.~J.
\newblock \emph{An introduction to the bootstrap}.
\newblock CRC press, 1994.

\bibitem[El-Yaniv \& Wiener(2010)El-Yaniv and Wiener]{El-Yaniv10o}
El-Yaniv, R. and Wiener, Y.
\newblock On the foundations of noise-free selective classification.
\newblock \emph{Journal of Machine Learning Research}, 11:\penalty0 1605--1641,
  2010.

\bibitem[{Flores}(1958)]{Flores58a}
{Flores}, I.
\newblock An optimum character recognition system using decision functions.
\newblock \emph{IRE Transactions on Electronic Computers}, EC-7\penalty0
  (2):\penalty0 180--180, 1958.

\bibitem[Friedman et~al.(2001)Friedman, Hastie, and Tibshirani]{Friedman01t}
Friedman, J., Hastie, T., and Tibshirani, R.
\newblock \emph{The elements of statistical learning}, volume~1.
\newblock Springer series in statistics New York, 2001.

\bibitem[Gal \& Ghahramani(2015)Gal and Ghahramani]{Gal15b}
Gal, Y. and Ghahramani, Z.
\newblock Bayesian convolutional neural networks with bernoulli approximate
  variational inference.
\newblock \emph{arXiv preprint arXiv:1506.02158}, 2015.

\bibitem[Gal \& Ghahramani(2016)Gal and Ghahramani]{Gal16d}
Gal, Y. and Ghahramani, Z.
\newblock Dropout as a bayesian approximation: Representing model uncertainty
  in deep learning.
\newblock In \emph{International Conference on Machine Learning}, pp.\
  1050--1059, 2016.

\bibitem[Geifman \& El-Yaniv(2017)Geifman and El-Yaniv]{Geifman17s}
Geifman, Y. and El-Yaniv, R.
\newblock Selective classification for deep neural networks.
\newblock In \emph{Advances in Neural Information Processing Systems}, pp.\
  4878--4887. 2017.

\bibitem[Goodfellow et~al.(2014)Goodfellow, Shlens, and Szegedy]{Goodfellow14e}
Goodfellow, I.~J., Shlens, J., and Szegedy, C.
\newblock Explaining and harnessing adversarial examples.
\newblock \emph{CoRR}, 2014.

\bibitem[Guo et~al.(2017)Guo, Pleiss, Sun, and Weinberger]{Guo17o}
Guo, C., Pleiss, G., Sun, Y., and Weinberger, K.~Q.
\newblock On calibration of modern neural networks.
\newblock In \emph{International Conference on Machine Learning}, pp.\
  1321--1330, 2017.

\bibitem[Hastie et~al.(2001)Hastie, Tibshirani, and Friedman]{Hastie01t}
Hastie, T., Tibshirani, R., and Friedman, J.
\newblock \emph{The Elements of Statistical Learning}.
\newblock Springer, 2001.

\bibitem[Hawkins(1980)]{Hawkins80i}
Hawkins, D.~M.
\newblock \emph{Identification of outliers}, volume~11.
\newblock Springer, 1980.

\bibitem[{He} et~al.(2016){He}, {Zhang}, {Ren}, and {Sun}]{He16d}
{He}, K., {Zhang}, X., {Ren}, S., and {Sun}, J.
\newblock Deep residual learning for image recognition.
\newblock In \emph{2016 IEEE Conference on Computer Vision and Pattern
  Recognition (CVPR)}, 2016.

\bibitem[Heek \& Kalchbrenner(2019)Heek and Kalchbrenner]{Heek19b}
Heek, J. and Kalchbrenner, N.
\newblock Bayesian inference for large scale image classification.
\newblock \emph{arXiv preprint arXiv:1908.03491}, 2019.

\bibitem[Hodge \& Austin(2004)Hodge and Austin]{Hodge04a}
Hodge, V. and Austin, J.
\newblock A survey of outlier detection methodologies.
\newblock \emph{Artificial intelligence review}, 22\penalty0 (2):\penalty0
  85--126, 2004.

\bibitem[{Huang} et~al.(2017){Huang}, {Liu}, v.~d. {Maaten}, and
  {Weinberger}]{Huang17d}
{Huang}, G., {Liu}, Z., v.~d. {Maaten}, L., and {Weinberger}, K.~Q.
\newblock Densely connected convolutional networks.
\newblock In \emph{2017 IEEE Conference on Computer Vision and Pattern
  Recognition (CVPR)}, 2017.

\bibitem[Kingma \& Ba(2017)Kingma and Ba]{Kingma17a}
Kingma, D.~P. and Ba, J.
\newblock Adam: A method for stochastic optimization.
\newblock \emph{arXiv preprint arXiv:1412.6980v9}, 2017.

\bibitem[Li et~al.(2018)Li, Wang, and Ding]{Li18r}
Li, H., Wang, X., and Ding, S.
\newblock Research and development of neural network ensembles: a survey.
\newblock \emph{Artificial Intelligence Review}, 49\penalty0 (4):\penalty0
  455--479, 2018.

\bibitem[MacKay(1992)]{MacKay92a}
MacKay, D. J.~C.
\newblock A practical bayesian framework for backpropagation networks.
\newblock \emph{Neural Comput.}, 4\penalty0 (3):\penalty0 448–472, 1992.

\bibitem[Madry et~al.(2018)Madry, Makelov, Schmidt, Tsipras, and
  Vladu]{Madry18t}
Madry, A., Makelov, A., Schmidt, L., Tsipras, D., and Vladu, A.
\newblock Towards deep learning models resistant to adversarial attacks.
\newblock In \emph{International Conference on Learning Representations}, 2018.

\bibitem[Mandelbaum \& Weinshall(2017)Mandelbaum and Weinshall]{Mandelbaum17d}
Mandelbaum, A. and Weinshall, D.
\newblock Distance-based confidence score for neural network classifiers.
\newblock \emph{arXiv preprint arXiv:1709.09844}, 2017.

\bibitem[Nagurney \& Zhang(2012)Nagurney and Zhang]{Nagurney12p}
Nagurney, A. and Zhang, D.
\newblock \emph{Projected Dynamical Systems and Variational Inequalities with
  Applications}.
\newblock Springer, 2012.

\bibitem[Neal(1996)]{Neal96b}
Neal, R.~M.
\newblock \emph{Bayesian Learning for Neural Networks}.
\newblock Springer, 1996.

\bibitem[Platt(1999)]{Platt99p}
Platt, J.~C.
\newblock Probabilistic outputs for support vector machines and comparisons to
  regularized likelihood methods.
\newblock In \emph{ADVANCES IN LARGE MARGIN CLASSIFIERS}, pp.\  61--74, 1999.

\bibitem[Rasmussen \& Williams(2005)Rasmussen and Williams]{Rasmussen05g}
Rasmussen, C. and Williams, C.
\newblock \emph{Gaussian Processes for Machine Learning}.
\newblock MIT Press, 2005.

\bibitem[Roberts \& Tesman(2009)Roberts and Tesman]{Roberts09a}
Roberts, F. and Tesman, B.
\newblock \emph{Applied Combinatorics}.
\newblock Chapman \& Hall/CRC, 2nd edition, 2009.

\bibitem[{Sheikholeslami} et~al.(2019){Sheikholeslami}, {Jain}, and
  {Giannakis}]{Sheikholeslami19e}
{Sheikholeslami}, F., {Jain}, S., and {Giannakis}, G.~B.
\newblock Efficient randomized defense against adversarial attacks in deep
  convolutional neural networks.
\newblock In \emph{IEEE International Conference on Acoustics, Speech and
  Signal Processing}, pp.\  3277--3281, 2019.

\bibitem[Shridhar et~al.(2019)Shridhar, Laumann, and Liwicki]{Shridhar19a}
Shridhar, K., Laumann, F., and Liwicki, M.
\newblock A comprehensive guide to bayesian convolutional neural network with
  variational inference.
\newblock \emph{arXiv preprint arXiv:1901.02731}, 2019.

\bibitem[Smithson(2002)]{Smithson02c}
Smithson, M.
\newblock \emph{Confidence intervals}, volume 140.
\newblock Sage Publications, 2002.

\bibitem[Szegedy et~al.(2014)Szegedy, Zaremba, Sutskever, Bruna, Erhan,
  Goodfellow, and Fergus]{Szegedy14i}
Szegedy, C., Zaremba, W., Sutskever, I., Bruna, J., Erhan, D., Goodfellow, I.,
  and Fergus, R.
\newblock Intriguing properties of neural networks.
\newblock In \emph{International Conference on Learning Representations}, 2014.

\bibitem[Tsipras et~al.(2019)Tsipras, Santurkar, Engstrom, Turner, and
  Madry]{Tsipras19r}
Tsipras, D., Santurkar, S., Engstrom, L., Turner, A., and Madry, A.
\newblock Robustness may be at odds with accuracy.
\newblock In \emph{International Conference on Learning Representations}, 2019.

\bibitem[Wong \& Kolter(2018)Wong and Kolter]{Wong18p}
Wong, E. and Kolter, Z.
\newblock Provable defenses against adversarial examples via the convex outer
  adversarial polytope.
\newblock In \emph{International Conference on Machine Learning}, pp.\
  5286--5295, 2018.

\bibitem[Xiao et~al.(2017)Xiao, Rasul, and Vollgraf]{Xiao17f}
Xiao, H., Rasul, K., and Vollgraf, R.
\newblock Fashion-{MNIST}: {A} novel image dataset for benchmarking machine
  learning algorithms.
\newblock \emph{arXiv preprint arXiv:1708.07747}, 2017.

\end{thebibliography}

\onecolumn

\section{Additional numerical experiments}

In these additional results, we repeat the experiments from the main paper on a different dataset~(namely, the fashion MNIST~\citep{Xiao17f}) to show how they carry over to another application. Here, we perform the experiments using the ResNet18~\citep{He16d} architecture trained over $8$~epochs in mini-batches of $256$~samples using Adam with the default parameters from~\citep{Kingma17a} and weight decay of~$10^{-3}$. Without data augmentation, the final classifier achieves an accuracy of~$92\%$ over the test set. Once again, the loss function~$\ell$ used is the cross-entropy loss and all experiments were performed over a~$1000$ images random sample from the Fashion MNIST test set. In the sequel, we take~$\bW = \gamma \bI$ with~$\gamma = 200$.

We once again begin leveraging the perturbation stability interpretation of~$\bc^\dagger$ to analyze the robustness of the ResNet on this new dataset by looking at the normalized values of~$\bc^o$ and~$\bc^\dagger$ sorted in decreasing order~(Figure~\ref{F:cred}). Notice that, in contrast to the ResNet trained on CIFAR10, the perturbed profiles here much noisier. This classifier is therefore less robust to perturbations of the input: its output can be considerably modified by comparably small perturbations. Notice that this analysis does not apply to the ResNet architecture in general, but to the specific instance used to classify these images. While these differences can be due to a larger sensitivity of the data (Fashion MNIST pictures are black-and-white whereas CIFAR10 has colored images), it can also be due to the fact that we trained for fewer epochs and did not use data augmentation. The power of using credibility profiles to analyze robustness is exactly due to the fact that it holds for the specific instance and application, in contrast to an average analysis.

Results for the filtering classifiers are shown in Figures~\ref{F:filtering_softmax_app} and~\ref{F:filtering_cred_app}. Once again, the model classifies samples only if the second largest credibility~(or softmax output entry) is less than the largest one by a factor of at least~$1-\alpha$, where~$\alpha$ is chosen to achieve specific coverage levels. We show results both using the softmax output of the model and the credibility profile~$\bc^\dagger$.

Once again, we notice a trade-off between robustness and performance, as in~\citep{Tsipras19r}. When classifying every sample~($100\%$~coverage), the classifier based on~$\bc^\dagger$ has lower accuracy~(Figure~\ref{F:accuracy_app}). However, it can discriminate between certain and uncertain predictions in a way that the softmax output of the model cannot~(as seen for~$\bc^o$ in Figure~\ref{F:cred}). In order to achieve the same accuracy as the unmodified model, the credibility-based filtered classifier must reduce its coverage to approximately~$80\%$. However, the overconfidence of the model output in its predictions makes it susceptible to perturbations. We illustrate this point by applying the different classifiers to test images corrupted by adversarial noise. The attacks shown here use the same parameters as described in the paper, except for the perturbation magnitudes that are twice as large. Still, we observe the same pattern: even for the weakest attack, the accuracy of the softmax output drops by almost~$60\%$~(see Figure~\ref{F:accuracy_app}), to the point that it would need to reduce its coverage by approximately~$40\%$ to recover the accuracy of the credibility-based classifier with full coverage~(Figure~\ref{F:filtering_softmax_app}).

\begin{figure}[tb]
\begin{minipage}[c]{0.48\columnwidth}
\centering
\includegraphics[width=\columnwidth]{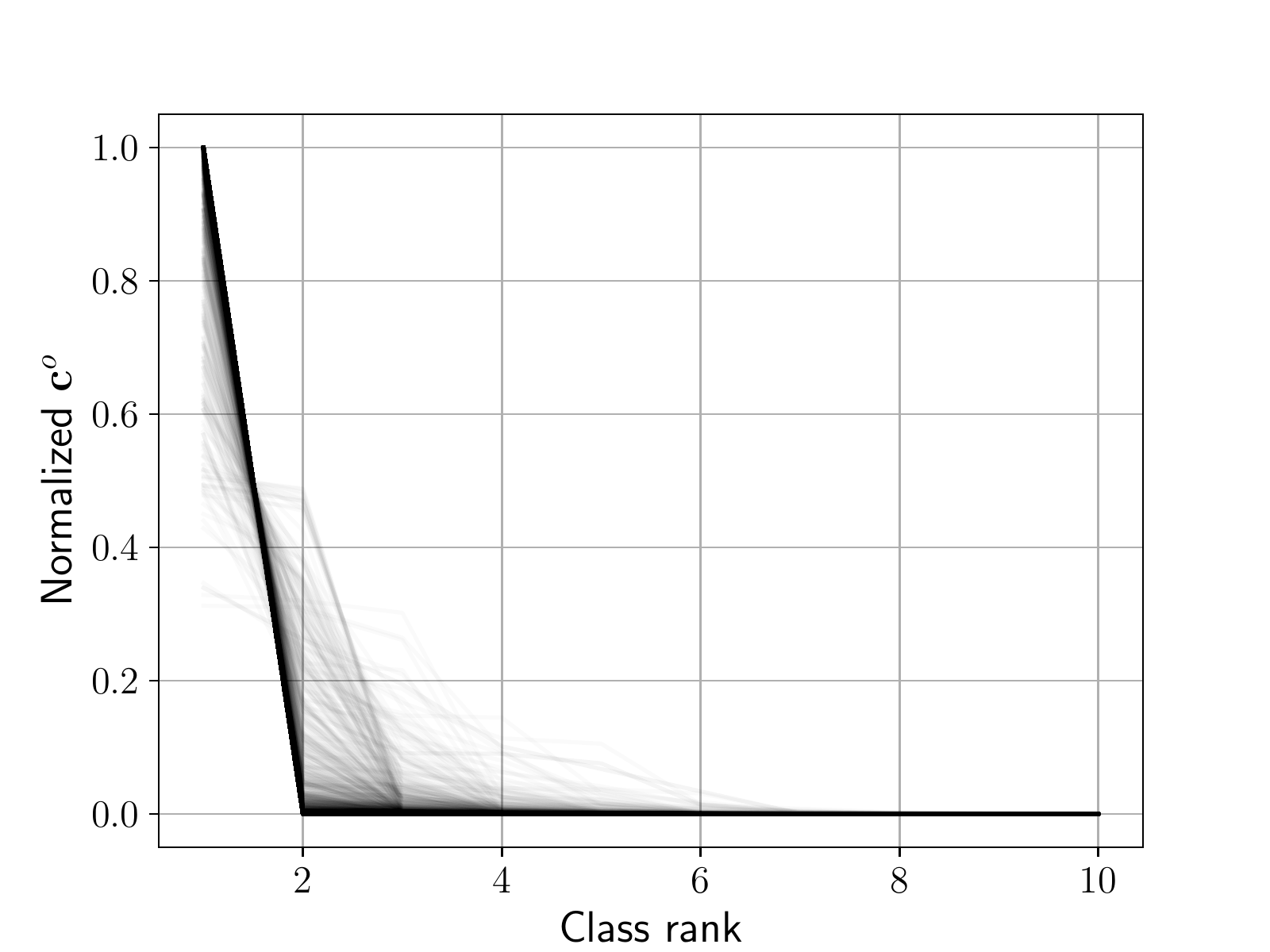}
\end{minipage}
\hfill
\begin{minipage}[c]{0.48\columnwidth}
\centering
\includegraphics[width=\columnwidth]{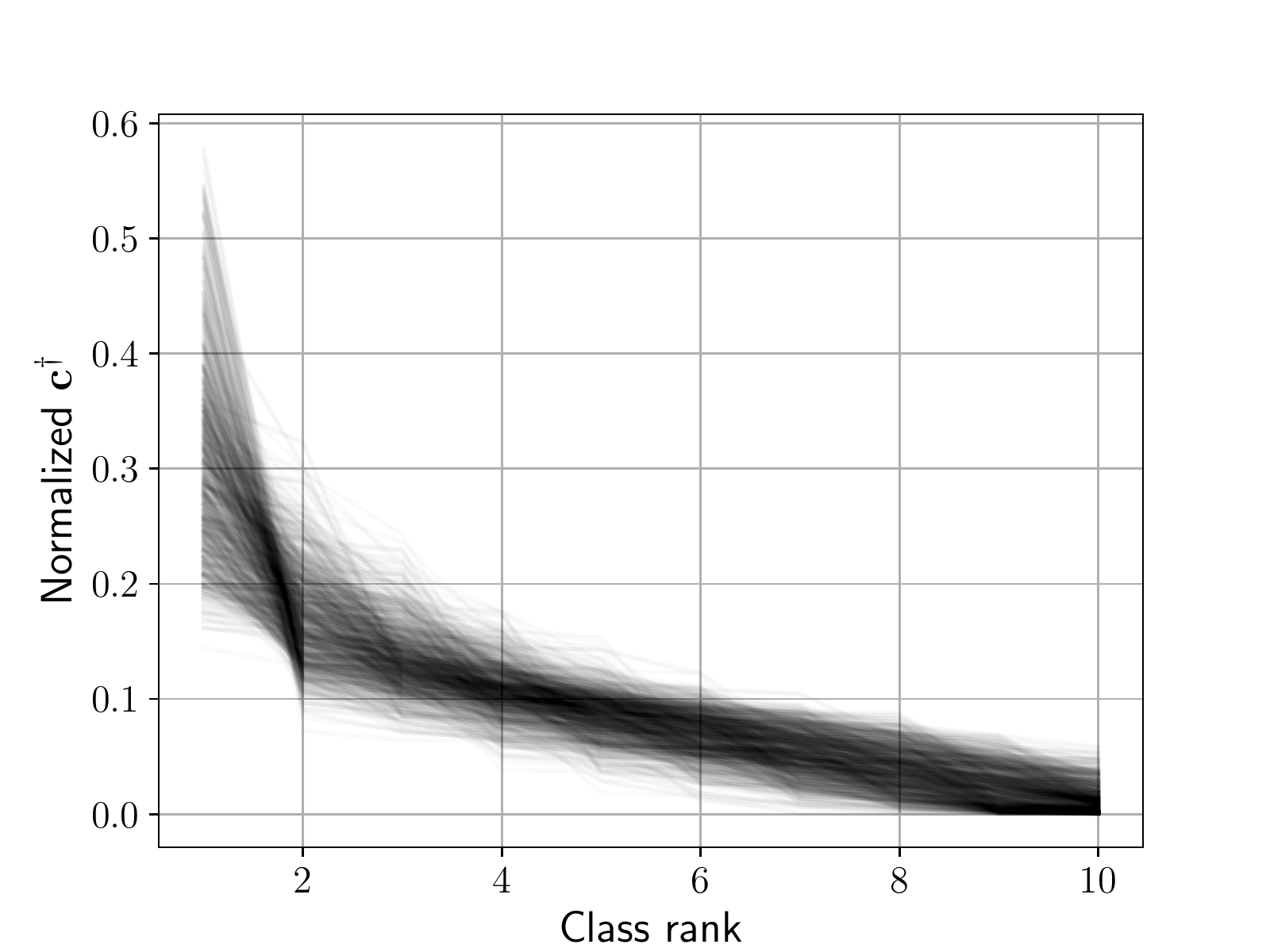}
\end{minipage}
\caption{Credibility profiles for ResNet18 on Fashion MNIST}
	\label{F:cred}
\end{figure}

\begin{figure}[tb]
\begin{minipage}[c]{0.48\columnwidth}
\centering
\includegraphics[width=\columnwidth]{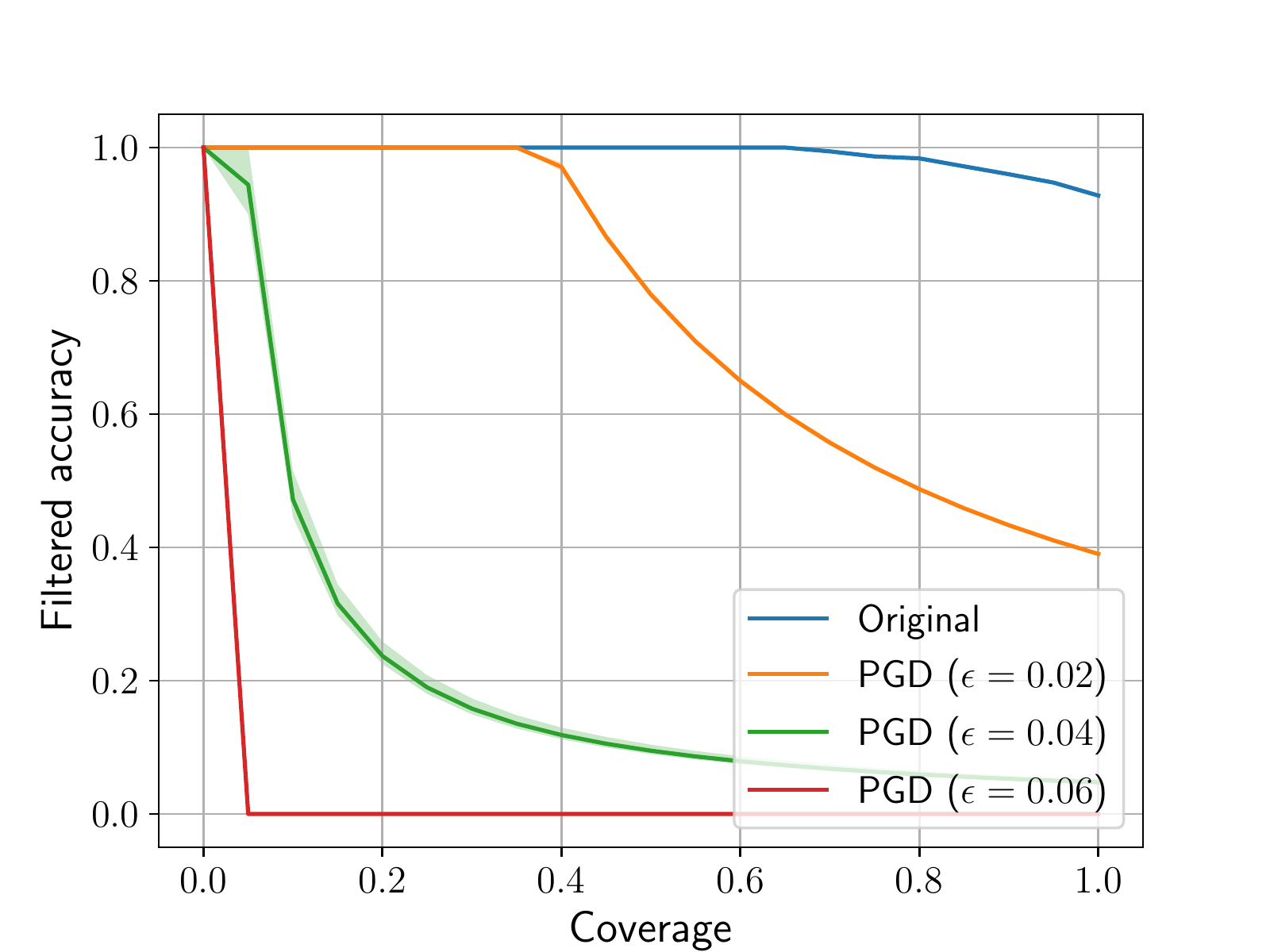}
\caption{Performance of softmax-based filtered classifier}
\label{F:filtering_softmax_app}
\end{minipage}
\hfill
\begin{minipage}[c]{0.48\columnwidth}
\centering
\includegraphics[width=\columnwidth]{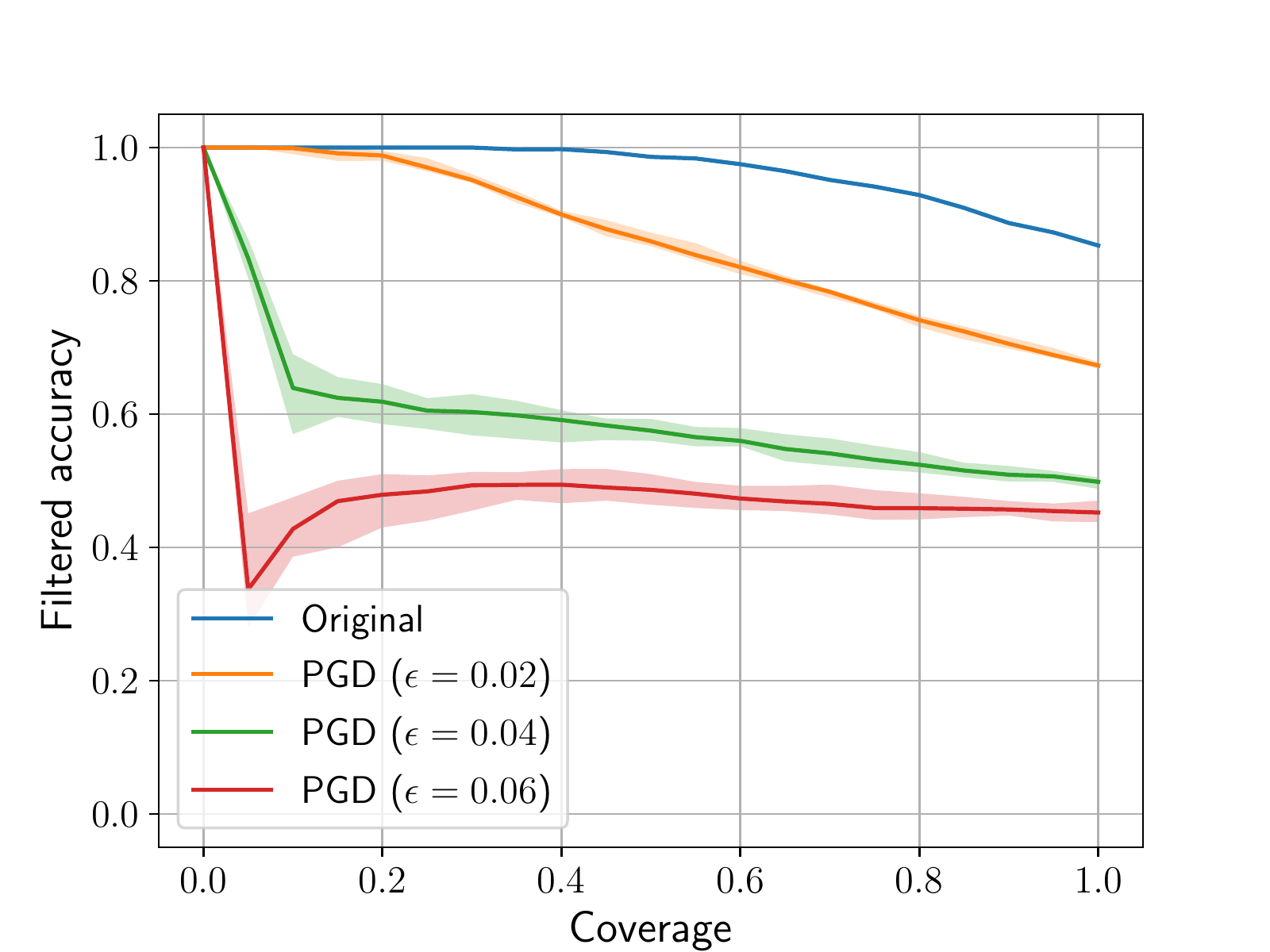}
\caption{Performance of credibility-based filtered classifier}
\label{F:filtering_cred_app}
\end{minipage}
\end{figure}

\begin{figure}[tb]
\centering
\includegraphics[width=0.48\columnwidth]{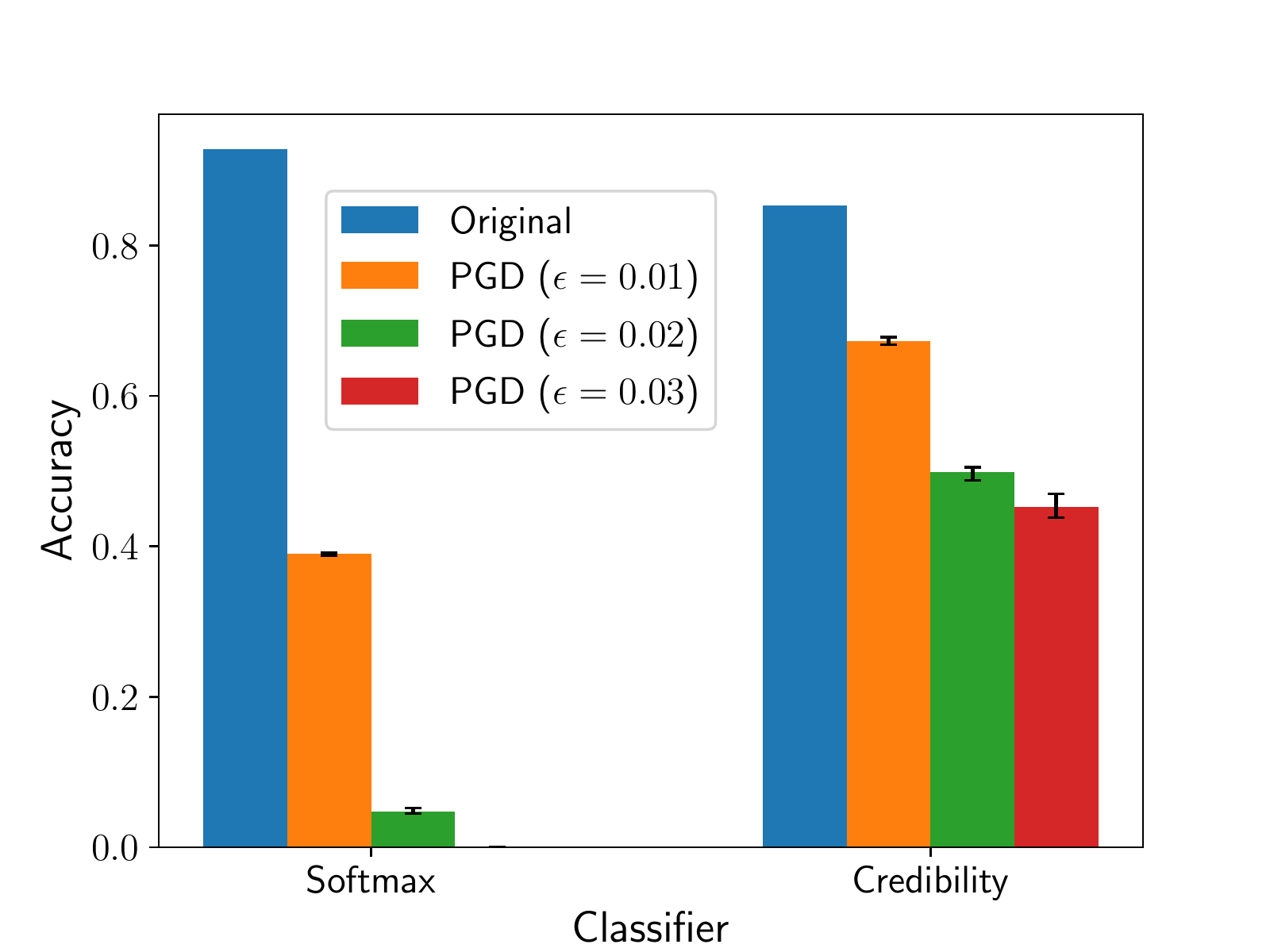}
\caption{Accuracy of softmax- and credibility-based classifiers}
\label{F:accuracy_app}
\end{figure}

\section{Proof of counterfactual results}
\label{X:counterfactual}

In this section, we prove different versions of the counterfactual result in Section~5.2. We begin with a simple lemma showing the existence of a local credibility profile satisfying~\eqref{E:compromise2}.

\begin{lemma}

There exists a local credibility profile~$\bc^\dagger$ that satisfies~\eqref{E:compromise2}.

\end{lemma}

\begin{proof}
Let~$(\bxb,\bcb)$ be a local minimizer pair of the optimization problem
\begin{prob}[\textup{P-A}]\label{P:both}
	\minimize_{\substack{\bx \in \setR^p \\ \bc \in\setR^\abs{\calK}}}&
		&&\norm{\bx - \bx^o}^2 + \norm{\bc}^2_{\bW^{-1}}
	\\
	\subjectto& &&\ell(\phi^o(\bx), k) \leq -c_k
		\text{,} \quad k \in \calK
		\text{.}
\end{prob}
If no credibility profile~$\bc$ satisfies~\eqref{E:compromise2}, then arbitrarily close to~$\bcb$ there exists~$(\bx^\prime,\bc^\prime)$ feasible for~(PI) such that
\begin{equation}\label{E:violation}
	\norm{\bxb - \bx^o}^2 + \norm{\bcb}^2_{\bW^{-1}} =
		r^\dagger(\bcb) + \norm{\bcb}^2_{\bW^{-1}} >
		\norm{\bx^\prime - \bx^o}^2 + \norm{\bc^\prime}^2_{\bW^{-1}}
		\text{.}
\end{equation}
However, observe that the feasibility set of~(PI) is contained in the feasibility set of~\eqref{P:both}, so that~$(\bx^\prime,\bc^\prime)$ is also~\eqref{P:both}-feasible. This leads to a contradiction since~\eqref{E:violation} violates the fact that~$(\bxb,\bcb)$ are local minimizers of~\eqref{P:both}.
\end{proof}

Thus, the existence of a local credibility profile is not an issue. Let us then complete the proof of Theorem~2.

\begin{proof}[Proof of Theorem~2]

Under Assumption~1 and for~$\bc \in \calC$ satisfying the hypothesis in~\eqref{E:counterfactual}, the~$\bx^\dagger(\bc)$ that achieves the local minimum~$r^\dagger(\bc)$ satisfies the KKT conditions~(Theorem~1) with~$\lambda_k = -2 w_k^{-1} c_k$. Explicitly,
\begin{equation}\label{E:cf_kkt}
	2 \big( \bx^\dagger(\bc) - \bx^o \big)
		- 2 \sum_{k=1}^\abs{\calK} \frac{c_k}{w_k}
			\nabla_{\bx} \ell_k\big( \phi^o\big(\bx^\dagger(\bc)\big)\big) = 0
	\quad \text{and} \quad
	\frac{2 c_k}{w_k} \left[ \ell_k\big( \phi^o\big(\bx^\dagger(\bc)\big)\big) + c_k \right] = 0
		\text{, for all } k \in \calK
		\text{.}
\end{equation}
Taking~$(\bx^\prime,\bc^\prime)$ to be any feasible pair for~(PI), i.e., $\ell_k(\phi^o(\bx^\prime)) \leq -c^\prime_k$ for all~$k \in \calK$, for~$\bx^\prime \in \calN$ as in Assumption~2, we can combine~\eqref{E:cf_kkt} to get
\begin{align*}
	r^\dagger(\bc) &= \norm{\bx^\dagger(\bc) - \bx^o}^2
		- 2 \sum_{k \in \calK} \frac{c_k}{w_k}
			\left[ \ell_k\big( \phi^o\big(\bx^\dagger(\bc)\big)\big) + c_k \right]
	\\
	{}&+ 2 \big( \bx^\prime - \bx^\dagger(\bc) \big)^T \left[
			\big( \bx^\dagger(\bc)-\bx^o \big)
			- \sum_{k \in \calK} \frac{c_k}{w_k}
				\nabla_{\bx} \ell_k\big( \phi^o\big(\bx^\dagger(\bc)\big)\big) \right]
		\text{,}
\end{align*}
which can then be rearranged to read
\begin{equation}\label{E:subgrads}
\begin{aligned}
	r^\dagger(\bc) &= \norm{\bx^\dagger(\bc) - \bx^o}^2
		+ 2\big( \bx^\prime - \bx^\dagger(\bc) \big)^T \big( \bx^\dagger(\bc) - \bx^o \big)
	\\
	{}&- 2\sum_{k \in \calK} \frac{c_k}{w_k} \bigg[ \ell_k\big( \phi^o\big(\bx^\dagger(\bc)\big)\big)
			+ \big( \bx^\prime - \bx^\dagger(\bc) \big)^T
				\nabla_{\bx} \ell_k\big( \phi^o\big(\bx^\dagger(\bc)\big)\big) + c_k \bigg]
		\text{.}
\end{aligned}
\end{equation}

To proceed, use the convexity of~$\norm{\cdot}^2$ to bound the first two terms in~\eqref{E:subgrads} and obtain
\begin{equation}\label{E:subgrads2}
	r^\dagger(\bc) \leq \norm{\bx^\prime - \bx^o}^2
		- 2 \sum_{k \in \calK} \frac{c_k}{w_k} \bigg[ \ell_k\big( \phi^o\big(\bx^\dagger(\bc)\big)\big)
			+ \big( \bx^\prime - \bx^\dagger(\bc) \big)^T
				\nabla_{\bx} \ell_k\big( \phi^o\big(\bx^\dagger(\bc)\big)\big) + c_k \bigg]
		\text{.}
\end{equation}
Then, since~$\bx^\prime \in \calN$ from Assumption~2, we can use~\eqref{E:condition}\ on the bracketed quantity in~\eqref{E:subgrads2} to write the inequality
\begin{equation}\label{E:subgrads3}
	r^\dagger(\bc) - \frac{1}{2} \norm{\bx^\prime - \bx^o}^2
		\leq - 2 \sum_{k \in \calK} \frac{c_k}{w_k} \bigg[
			\ell_k\big( \phi^o(\bx^\prime) \big)
			+ \frac{\left[ \ell_k\big( \phi^o(\bx^\prime) \big)
					- \ell_k\big( \phi^o\big(\bx^\dagger(\bc)\big)\big) \right]^2
			}{
				2 \ell_k\big( \phi^o\big(\bx^\dagger(\bc)\big)\big)
			}
			+ c_k \bigg]
		\text{.}
\end{equation}
Expanding~\eqref{E:subgrads3} yields
\begin{equation*}
	r^\dagger(\bc) - \norm{\bx^\prime - \bx^o}^2 \leq
		- \sum_{k \in \calK} \frac{
			c_k \ell_k\big( \phi^o(\bx^\prime) \big)^2
		}{
			w_k \ell_k\big( \phi^o\big(\bx^\dagger(\bc)\big)\big)
		}
		- \sum_{k \in \calK}
			\frac{c_k \ell_k\big( \phi^o\big(\bx^\dagger(\bc)\big)\big)}{w_k}
		- 2 \norm{\bc}_{\bW^{-1}}^2
		\text{,}
\end{equation*}
where we used the fact that~$\norm{\bc}_{\bW^{-1}}^2 = \sum_k w_k^{-1} c_k^2$. Since~$\ell_k\big( \phi^o\big(\bx^\dagger(\bc)\big)\big) \leq -c_k$ and~$\ell_k\big( \phi^o\big(\bx^\dagger(\bc^\prime)\big)\big) \leq -c_k^\prime$, we get
\begin{equation*}
	r^\dagger(\bc) - \norm{\bx^\prime - \bx^o}^2 \leq
		\norm{\bc^\prime}_{\bW^{-1}}^2 - \norm{\bc}_{\bW^{-1}}^2
		\text{.}
\end{equation*}
Hence, if~$\bW^{-1} \bc$ is a dual variable of~(PI), then~$\bc$ satisfies~\eqref{E:compromise2}, i.e., $\bc = \bc^\dagger$.
\end{proof}

In the particular case in which~(PI) is a convex program, we can show that~\eqref{E:counterfactual}\ is both sufficient and necessary:

\begin{proposition}

Assume that~$\ell_k$ is convex and non-decreasing for all~$k \in \calK$ and~$\phi^o$ is a convex function of~$\bx$~(e.g., linear). Let~$r^\dagger(\bc)$ be a local minimum of~(PI) for~$\bc \in \calC$ achieved by~$\bx^\dagger(\bc)$ with value associated with the dual variables~$\blambda^\dagger(\bc)$. Then, under Assumption~1, the local credibility profile~$\bc^\dagger$ from~\eqref{E:compromise2}\ is a global profile satisfying~\eqref{E:compromise}\ and
\begin{equation}\label{E:cxvCounterfactual}
	\bc = -\frac{1}{2} \bW \blambda^\star(\bc)
	\Leftrightarrow
	\bc \text{ satisfies~\eqref{E:compromise}, i.e., } \bc = \bc^\star
		\text{.}
\end{equation}

\end{proposition}

\begin{proof}

The first part of the theorem is immediate: since~$\phi$ is convex and~$\ell$ is a non-decreasing convex function, (PI) is a strongly convex problem. Hence, any local minimizer~$\bx^\dagger$ is a global minimizer~$\bx^\star$~\citep{Boyd04c}.

We proceed by proving necessity~($\Rightarrow$). Since~(PI) is a strongly convex, strictly feasible problem~[$\bc \in \calC$ in~\eqref{E:slater}], it is strongly dual and its dual variables~$\blambda^\star(\bc)$ is a subgradient of its perturbation function~$r^\star$~\citep[Section 5.6.2]{Boyd04c}. We therefore obtain
\begin{equation}\label{E:convex_inequalities}
\begin{aligned}
	r^\star(\bc_0) &\geq r^\star(\bc) + \blambda^\star(\bc)^T (\bc_0 - \bc)
	\\
	\norm{\bc_0}^2_{\bW^{-1}} &\geq \norm{\bc}^2_{\bW^{-1}}
		+ 2 \bc^T \bW^{-1} (\bc_0 - \bc)
		\text{,}
\end{aligned}
\end{equation}
which when summed read
\begin{equation}\label{E:summed_inequalities}
	r^\star(\bc) - r^\star(\bc_0) \leq
		\norm{\bc_0}^2_{\bW^{-1}} - \norm{\bc}^2_{\bW^{-1}}
		+ \left[ \blambda^\star(\bc) + 2 \bW^{-1} \bc \right]^T (\bc_0 - \bc)
		\text{.}
\end{equation}
Using the hypothesis in~\eqref{E:cxvCounterfactual} yields~$\blambda^\star(\bc) + 2 \bW^{-1} \bc = \bzero$ and~\eqref{E:summed_inequalities} reduces to~\eqref{E:compromise}, which implies~$\bc = \bc^\star$.

To prove the sufficiency part~($\Leftarrow$) in~\eqref{E:cxvCounterfactual}, notice from~\eqref{P:both} that any~$\bc^\star$ that satisfies~\eqref{E:compromise}\ is the minimizer of the strongly convex function~$q(\bc) = r^\star(\bc) + \norm{\bc}_{\bW^{-1}}^2$. As such, it must be that
\begin{equation*}
	\nabla q(\bc^\star) = \bzero \Leftrightarrow \nabla r^\star(\bc^\star) + 2 \bW^{-1} \bc^\star = \bzero
		\text{.}
\end{equation*}
Once again leveraging the fact that~(PI) is strongly dual, it holds that~$\blambda^\star(\bc) = \nabla p^\star(\bc)$ for all~$\bc$~\citep[Section 5.6.3]{Boyd04c}, thus concluding the proof.
\end{proof}


\section{Proof of Proposition~1}
\label{X:map}

Recall the Bayesian formulation from Section~6:
\begin{align*}
	\Pr\left( \bx \mid \bc \right) &=
		\calN\left( \bx \mid \bx^o, (2t)^{-1} \bI \right)
		\times \prod_{c_k \neq 0}
			\text{SE}\left[ \ell(\phi^o(\bx), k) \mid c_k, \frac{2 c_k}{w_k} t \right]
	\\
	\Pr\left( \bc \right) &= \calN\left( \bc \mid \zeros, (2t)^{-1} \bW \right)
\end{align*}
We then obtain the joint distribution~$\Pr(\bx,\bc)$ as
\begin{equation}\label{E:joint}
\begin{aligned}
	 \Pr(\bx,\bc) &\propto
	 	\exp\left( -t \norm{\bx - \bx^o}^2 \right) \times
	 	\prod_{c_k \neq 0} \frac{2 c_k}{w_k} t
	 		\exp\bigg( -\frac{2 c_k}{w_k} t \left[ \ell_k(\phi^o(\bx) - c_k \right] \bigg)
	 	\exp\left( - t \norm{\bc}^2_{\bW^{-1}} \right)
	 \\
	 {}&= \exp\left(
	 	-t \norm{\bx - \bx^o}^2
	 	-t \sum_{c_k \neq 0} \frac{2c_k}{w_k} \left[ \ell_k(\phi^o(\bx) - c_k \right]
		- t \norm{\bc}^2_{\bW^{-1}}
	 	+ \sum_{c_k \neq 0} \log\left( \frac{2c_k}{w_k} t \right)
	 \right)
		 \text{.}
\end{aligned}
\end{equation}

Notice that since~$\bx^\dagger(\bc^\dagger)$ achieves the local minimum~$r^\dagger(\bc^\dagger)$ that satisfies Def.~2, the pair~$\big( \bx^\dagger(\bc^\dagger), \bc^\dagger \big)$ is a local minimizer of~\eqref{P:both}. Indeed, \eqref{E:compromise2}\ can be rearranged as in
\begin{equation*}
	\norm{\bx^\dagger(\bc^\dagger) - \bx^o}^2 + \norm{\bc^\dagger}_{\bW^{-1}}^2
		\leq \norm{\bx^\prime - \bx^o}^2 + \norm{\bc^\prime}_{\bW^{-1}}^2
		\text{,}
\end{equation*}
for~$\bx^\prime$ in a neighborhood of~$\bx^\dagger(\bc^\dagger)$ and~$(\bx^\prime,\bc^\prime)$ (PI)-feasible. Hence, they satisfy the KKT conditions~(Theorem~1) for~\eqref{P:both}. Indeed, note that~\eqref{P:both} always has a strictly feasible pair~$(\bxh,\bch)$ obtained by taking~$[\bch]_k = \ell_k\big( \phi^o(\bx) \big)$. Explicitly, there exists~$\blambda$ such that
\begin{align}
	2 \big( \bx^\dagger-\bx^o \big)
		+ \sum_{k=1}^\abs{\calK} \lambda_k
			\nabla_{\bx} \ell_k\big( \phi^o\big(\bx^\dagger(\bc^\dagger)\big)\big) &= 0
		\label{E:grad_x}
		\text{,}
	\\
	2 \bW^{-1} \bc^\dagger + \blambda &= 0
		\label{E:grad_c}
		\text{,}
	\\
	\lambda_k \left[ \ell_k\big( \phi^o\big(\bx^\dagger(\bc^\dagger)\big)\big) + c_k \right] &= 0
		\text{, for all } k \in \calK
		\label{E:slackness2}
		\text{.}
\end{align}
Observe that~\eqref{E:grad_x} and~\eqref{E:grad_c} arise by applying~\eqref{E:kkt_stationary}\ taking derivatives with respect to~$\bx$ and~$\bc$ respectively and~\eqref{E:slackness2} is the complementary slackness condition~\eqref{E:kkt_comp_slack}. Using~\eqref{E:grad_c} and~\eqref{E:slackness2}, we additionally conclude that
\begin{equation}\label{E:vanishes}
	\sum_{k \in \calK} \frac{2c_k}{w_k} \left[ \ell_k\big( \phi^o\big(\bx^\dagger(\bc^\dagger)\big)\big) + c_k \right] = 0
		\text{.}
\end{equation}
Using~\eqref{E:vanishes}, the joint probability distribution~\eqref{E:joint} evaluated at~$(\bx^\dagger(\bc^\dagger), \bc^\dagger)$ reduces to
\begin{equation}
\begin{aligned}
	 \Pr\big( \bx^\dagger(\bc^\dagger),\bc^\dagger \big) &\propto \exp\left[
	 	-t \norm{\bx^\dagger(\bc^\dagger) - \bx^o}^2
		- t \norm{\bc^\dagger}^2_{\bW^{-1}}
	 	+ \sum_{c_k \neq 0} \log\left( \frac{2c_k}{w_k} t \right)
	 \right]
		 \text{.}
\end{aligned}
\end{equation}

Suppose now that there exists another point~$(\bx^\prime,\bc^\prime)$ arbitrarily close to~$\big( \bx^\dagger(\bc^\dagger), \bc^\dagger \big)$ such that~$\Pr\big(\bx^\dagger(\bc^\dagger),\bc^\dagger\big) < \Pr(\bx^\prime,\bc^\prime)$ for all~$t$. This would imply that
\begin{equation*}
	\norm{\bx^\dagger(\bc^\dagger) - \bx^o}^2 + \norm{\bc^\dagger}^2_{\bW^{-1}}
		- \frac{1}{t} \sum_{c_k \neq 0} \log\left( \frac{2c_k}{w_k} t \right)
	>
	\norm{\bx^\prime - \bx^o}^2 + \norm{\bc^\prime}^2_{\bW^{-1}}
		- \frac{1}{t} \sum_{c_k^\prime \neq 0} \log\left( \frac{2c^\prime_k}{w_k} t \right)
\end{equation*}
for all~$t$ and since the last term is~$o(t)$, we eventually get
\begin{equation*}
	\norm{\bx^\dagger(\bc^\dagger) - \bx^o}^2
		+ \norm{\bc^\dagger}^2_{\bW^{-1}}
	>
	\norm{\bx^\prime - \bx^o}^2 + \norm{\bc^\prime}^2_{\bW^{-1}}
\end{equation*}
which violates~\eqref{E:compromise2}. Hence, $\big( \bx^\dagger(\bc^\dagger), \bc^\dagger \big)$ becomes a local maximum of the joint distribution~\eqref{E:joint} as~$t \to \infty$.
\hfill$\square$

\end{document}